\documentclass[11pt]{article}
\pdfoutput=1
\usepackage[margin=1in]{geometry}
\usepackage{amsmath,amssymb,amsthm}
\usepackage{graphicx}
\usepackage{booktabs}
\usepackage{multirow}
\usepackage{url}
\usepackage[numbers]{natbib}
\usepackage[colorlinks=true,linkcolor=blue,citecolor=blue,urlcolor=blue]{hyperref}
\usepackage{microtype}
\usepackage{enumitem}

\newtheorem{theorem}{Theorem}

\newtheorem{assumption}{Assumption}
\theoremstyle{remark}
\newtheorem{remark}{Remark}

\newcommand{\E}{\mathbb{E}}
\newcommand{\Var}{\mathrm{Var}}
\newcommand{\cS}{\mathcal{S}}
\newcommand{\cC}{\mathcal{C}}
\newcommand{\cT}{\mathcal{T}}
\newcommand{\cA}{\mathcal{A}}
\newcommand{\cE}{\mathcal{E}}
\newcommand{\cF}{\mathcal{F}}
\newcommand{\yhat}{\hat{y}}
\newcommand{\Vhat}{\widehat{V}}
\newcommand{\Nhat}{\widehat{N}}

\title{Error Certificates for KV-Cache Eviction\\
via Randomized Design}

\author{Peng~Xie and Amr~Alanwar%
\thanks{Peng Xie and Amr Alanwar are with the TUM School of Computation, Information and Technology, Department of Computer Engineering, Technical University of Munich, 74076 Heilbronn, Germany (e-mail: p.xie@tum.de; alanwar@tum.de).}}
\date{August 2026}

\begin{document}
\maketitle

\begin{abstract}
Deterministic KV-cache eviction keeps the top-$k$ tokens under an
importance score and deletes the rest, and after the deletion the
serving system cannot know what the eviction cost it on the current
query. We replace the deterministic tail with Poisson sampling at known
inclusion probabilities, which makes the eviction error identifiable
and turns a survey-sampling variance estimator over the retained set
into a per-step error certificate at one extra scalar per retained
token. On a thirty-turn assistant compressed to a 10\% cache budget, the
certificate-gated system answers 0.97 of recall questions against 0.09
for top-$k$, and for facts stated 26 to 30 turns earlier it recalls
97\% against 2\%. We prove that no estimator computable from the
information a deterministic scheme retains is consistent for its own
eviction error: evicted values can be altered so that everything
retained is unchanged while the true attention-output error grows
without bound. Under the Poisson design the certificate covers the
realized attention error in 96.9--97.7\% of 12{,}096 replay cells and
in 98.1--99.7\% on twelve further architectures. Randomization buys attribution, not prediction: a pre-registered
study on LongBench at 6k and 16k tokens (about 74{,}000 generations)
finds question-aware eviction at 25--50\% budgets nearly free and output
log-probability the better failure predictor, while the certificate
answers the question confidence cannot, separating eviction-induced
from inherent failures at AUC 0.65--0.75 against 0.47--0.54, and
schedules recomputation at 1.7--1.8 times the gain of random gating. On
real long-term conversations the gated system returns the full-cache
score inside the heavy-damage regime, and the rule that triggers it is
the same across five model families.
\end{abstract}

\section{Introduction}
\label{sec:intro}

Long-context inference stores a key--value pair per token per layer, so
the KV cache grows linearly with context and soon dominates memory. The
standard remedy is eviction: score each cached token by an importance
proxy, keep the top $k$, and delete the rest permanently. A large
literature refines the score \citep{zhang2023h2o, li2024snapkv,
xiao2024streamingllm, cai2024pyramidkv, feng2024adakv, zhang2024cam,
momentkv2026}, evaluated by quality at matched budgets.

We ask a question the score race leaves open: after eviction, can the
serving system know what the eviction cost it on the current query? For
deterministic selection the answer is no. Top-$k$ keeps a set that is a
function of scores, so conditioned on what is retained the evicted
values are unconstrained: altering them changes the true attention
output arbitrarily while every retained key, value, score, and
downstream statistic stays bit-for-bit identical
(Theorem~\ref{thm:impossibility}). Any self-diagnostic computed from the
retained state reads the same whether the eviction was harmless or
destroyed the answer. We call this property \emph{silent failure} and
measure it at scale: every deterministic self-signal we test sits at
chance on its own eviction-induced failures. Figure~\ref{fig:hero} shows
it in one dialogue: thirty turns in, an assistant compressed to 15\% of
its cache invents the office wifi password at 97.7\% confidence, while
the certified path red-flags its own invented first answer and
retrieves the truth; at a 10\% budget the gated system answers 29 of
30 recall questions across five such dialogues, top-$k$ 1 of 30.

\begin{figure}[t]
\centering
\includegraphics[width=0.9\linewidth]{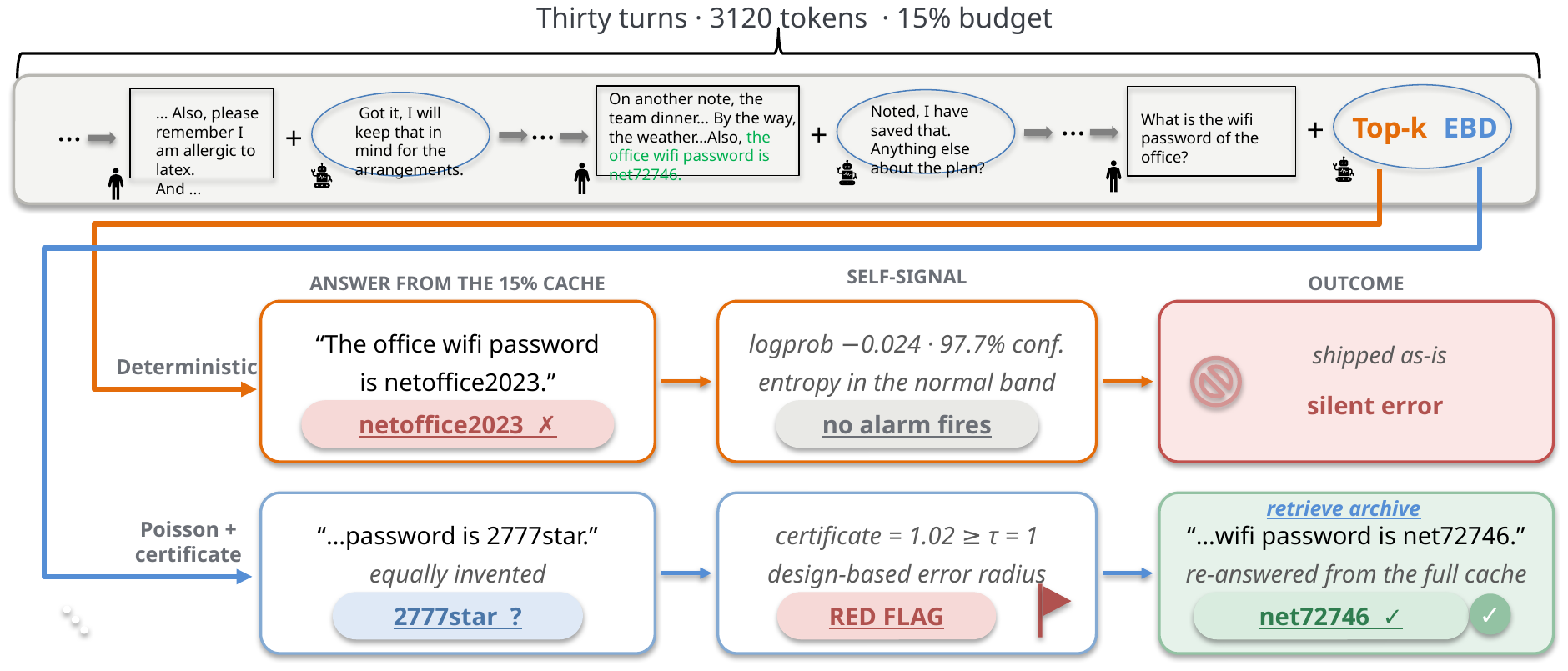}
\caption{At a 15\% budget, top-$k$ invents the wifi password at 97.7\%
confidence while the certified design red-flags its own first answer
and retrieves the truth (verbatim run; Qwen2.5-7B-Instruct; thirty-turn
system of Section~\ref{sec:vignettes}). Top: the conversation timeline;
turn 21 states the password (highlighted) and the question arrives ten
turns later, after eviction under deterministic top-$k$ and the
certified Poisson design (labeled EBD, eviction by design). Bottom:
each rule's answer, self-signal, and outcome.}
\label{fig:hero}
\end{figure}

The escape comes from survey statistics, not from a better score. If
the tail of the cache is evicted by Poisson sampling with inclusion
probabilities $\pi_i$ that the algorithm itself chooses, the design is
known and design-based inference applies \citep{horvitz1952, sen1953,
yates1953, sarndal1992}: an inverse-probability (H\'ajek) correction
removes the selection bias of the retained softmax, and a
Sen--Yates--Grundy variance estimator, computable from the retained set
alone, is unbiased for the variance of the linearized error
(Theorem~\ref{thm:identifiability}). The correction is one logit
offset, $\log(1/\pi_i)$, per retained tail token. An
empirical-Bernstein radius \citep{howard2021, waudbysmith2024} turns the
variance estimate into a per-step certificate, and an e-process
construction \citep{ramdas2023, vovk2021} extends it toward validity
over the whole autoregressive trajectory (Section~\ref{sec:cert}).

A pre-registered study, seven falsifiable claims with kill conditions
committed before the runs, locates where the certificate pays
(Section~\ref{sec:real}). Question-aware eviction at 25--50\% budgets
is nearly free on real tasks, output log-probability predicts failure
better than any cache-side signal, and the damage regime lives in
streaming and agent-memory settings, where history is compressed
before the question is known. There the certificate answers what
confidence cannot, whose fault a failure is, and that attribution
schedules recomputation better than random or confidence gating.
Randomization buys identifiability of the compression channel:
attribution, not prediction. Our contributions are the following.
\begin{itemize}[leftmargin=*,itemsep=1pt,topsep=2pt,parsep=0pt]
\item We prove that no estimator measurable with respect to the online
information of a deterministic, value-blind eviction scheme is
consistent for its own attention-output error
(Theorem~\ref{thm:impossibility}), and realize it in a real model,
where permuting evicted values changes the true error of top-$k$ by a
factor of 64 while every online statistic stays bit-identical
(Section~\ref{sec:replay}).
\item We give the design that restores identifiability,
certainty-plus-Poisson eviction with the H\'ajek correction as one
logit offset and an unbiased retained-set variance estimator
(Theorem~\ref{thm:identifiability}), whose per-step certificate covers
the realized error in 96.9--97.7\% of 12{,}096 replay cells at Spearman 0.94--0.98 and in 98.1--99.7\% of cells on twelve further architectures
spanning QK-norm, fused projections, soft-capping and mixture-of-experts
layers (Table~\ref{tab:zoo}).
\item We measure silent failure on a panel of four online self-signals
over thousands of scored generations: entropy and margin of
deterministic evictors sit at chance on their own eviction-induced
failures while the certificate reaches 0.77--0.86
(Section~\ref{sec:panel}).
\item A pre-registered study at 6k and 16k tokens (about 74{,}000
generations), agent-memory systems on five model families, real
conversations, and a 32B replication locate the value: attribution at
AUC 0.65--0.75 against 0.47--0.54 for output confidence, recomputation
scheduling at 1.7--1.8 times the gain of random gating, gated accuracy 0.97 against 0.09 for top-$k$ at a 10\% budget, 0.78--1.00 of the full-cache score on real long-term conversations at a resident cache of 400--1{,}000 tokens against 0.19--0.37 for top-$k$, and certificate-guided partial retrieval at 0.70 accuracy against 0.18 for
random blocks (Sections~\ref{sec:real} and~\ref{sec:vignettes}).
\end{itemize}

\section{Related work}
\label{sec:related}

\paragraph{Deterministic eviction and compensation.}
H2O \citep{zhang2023h2o}, SnapKV \citep{li2024snapkv}, TOVA
\citep{oren2024tova}, StreamingLLM \citep{xiao2024streamingllm},
PyramidKV \citep{cai2024pyramidkv}, and Ada-KV \citep{feng2024adakv}
select retained tokens deterministically from an importance proxy and
differ in the proxy or the budget allocation; CAOTE
\citep{goel2025caote} folds value information into the ranking, which
improves which tokens go, not what the system knows afterwards
(Remark~\ref{rem:valueaware}). Merging methods such as CaM
\citep{zhang2024cam} and MomentKV \citep{momentkv2026} fold evicted
mass into retained representatives, a post-hoc stratified ratio
correction of the point estimate. None maintains an estimate of its own
induced error, and Theorem~\ref{thm:impossibility} shows none can.

\paragraph{Randomized eviction and sparse attention with guarantees.}
MagicPIG \citep{chen2025magicpig} samples keys by locality-sensitive
hashing and corrects with self-normalized importance sampling, a point
estimator without a variance estimate. Nexus sampling \citep{nexus2026}
evicts by reservoir sampling and proves offline that a
Horvitz--Thompson estimator of retained utility is unbiased, but the
inclusion probabilities do not correct the attention computation and no
online error estimate is produced. VASE \citep{vase2026} adds
stochasticity to protect large-magnitude values, without probability
bookkeeping. vAttention \citep{desai2025vattention} and Quest
\citep{tang2024quest} give per-step guarantees or query-adaptive page
selection while keeping the full KV resident; permanent eviction cannot recompute a deleted token. A
fixed-contract diagnostic \citep{fixedcontract2026} studies when
value-aware selection helps, without a randomized design.

\paragraph{Design-based inference and anytime validity.}
The estimator layer is classical: Horvitz--Thompson estimation
\citep{horvitz1952}, the Sen--Yates--Grundy variance form
\citep{sen1953, yates1953}, Poisson and rejective sampling
\citep{hajek1964}, optimal allocation \citep{neyman1934}, and the design
effect \citep{kish1965}, consolidated in \citet{sarndal1992}. The
anytime layer is modern: time-uniform confidence sequences
\citep{howard2021}, empirical-Bernstein betting bounds
\citep{waudbysmith2024}, and e-values with their combination
\citep{ramdas2023, vovk2021}. Neither has been applied to permanent KV
eviction.

\section{Setup and theory}
\label{sec:theory}

\subsection{Objects}

Fix one attention head at decode step $t$; everything extends by
averaging over heads and layers. The cache is $\cC = \{(k_i, v_i)\}_{i=1}^{n}$,
the query is $q_t$, scores are $s_i = q_t^\top k_i/\sqrt{d}$, weights
$a_i = e^{s_i}$, and the full-cache head output is the self-normalized
mean
\begin{equation}
y_t \;=\; \frac{\sum_{i \in \cC} a_i v_i}{\sum_{i \in \cC} a_i}.
\label{eq:full}
\end{equation}
An \emph{eviction mechanism} maps the prefill state to a retained set
$\cS \subseteq \cC$; tokens outside $\cS$ are deleted permanently. The
\emph{online information} $\cF_t$ available to the serving system at step
$t$ is the retained KV pairs, the design (any sampling probabilities the
mechanism used), the query history, and every quantity computed from
these. Evicted values are not $\cF_t$-measurable. The estimation target
is the error $\|\yhat_t - y_t\|$ of whatever output $\yhat_t$ the system
computes from $\cS$.

\subsection{Impossibility for deterministic eviction}

\begin{theorem}[Unidentifiability]
\label{thm:impossibility}
Let the eviction mechanism be deterministic and value-blind, meaning
$\cS$ is a function of the keys, scores, and query history only. Then for
every $\cF_t$-measurable estimator $\widehat{E}$ and every $M > 0$ there
exist two cache configurations that generate identical online information
$\cF_t$ and whose true errors $\|\yhat_t - y_t\|$ differ by more than
$M$. Consequently no $\cF_t$-measurable estimator of the eviction error
is consistent, uniformly over cache configurations.
\end{theorem}

\begin{proof}
Fix any configuration and let $\cE = \cC \setminus \cS$ be the evicted
set, which is nonempty whenever eviction occurs. Replace $\{v_i\}_{i \in
\cE}$ by $\{v_i + c u\}_{i \in \cE}$ for a unit vector $u$ and scalar
$c$. Selection does not depend on values, so $\cS$ is unchanged; retained
pairs, scores, design, and query history are unchanged; hence $\cF_t$ and
with it $\widehat{E}$ and $\yhat_t$ are unchanged. The full output
\eqref{eq:full} shifts by $c u \cdot \sum_{i\in\cE} a_i / \sum_{i\in\cC}
a_i$, which is nonzero and scales linearly in $c$. Choosing $c$ large
makes the two errors differ by more than $M$.
\end{proof}

\begin{remark}[Value-aware scores]
\label{rem:valueaware}
If the score reads a finite vector of scalar summaries of each value,
for instance its norm \citep{vase2026, fixedcontract2026}, the
construction survives with summary-preserving perturbations: rotations
of evicted values preserve norms while moving $\sum_{i\in\cE} a_i v_i$
freely on a sphere whose radius the retained set cannot see.
\end{remark}

\subsection{Design: certainty plus Poisson tail, H\'ajek by logit offset}
\label{sec:design}

Existing randomized evictors discard the design after sampling
\citep{nexus2026, vase2026} or correct the point estimate without a
variance layer \citep{chen2025magicpig}. Keeping the inclusion
probabilities changes what is computable: a Poisson-sampled tail at
stored $\pi_i$ supports an unbiased output correction, an unbiased
variance estimator, and a certificate, from the retained set alone.

The mechanism keeps a \emph{certainty set} $\cA$ (score top slice plus
attention sinks and a recency window; $\pi_i = 1$) and applies to the
tail $\cT = \cC \setminus \cA$ independent Bernoulli retention with
inclusion probabilities
\begin{equation}
\pi_i \;=\; \mathrm{clip}\!\left(m \cdot
\frac{\mathrm{score}_i}{\sum_{j \in \cT}\mathrm{score}_j},\;
\varepsilon,\; 1\right), \qquad i \in \cT,
\label{eq:pi}
\end{equation}
where $m$ is the expected tail budget and $\varepsilon > 0$ a floor
(Poisson sampling \citep{hajek1964}). The output on the retained set is
the H\'ajek estimator
\begin{equation}
\yhat_t \;=\; \frac{\sum_{i \in \cS} (a_i/\pi_i)\, v_i}
{\sum_{i \in \cS} a_i/\pi_i},
\label{eq:hajek}
\end{equation}
which one line of code implements exactly: add $\log(1/\pi_i)$ to the
retained logit and let the softmax denominator do the renormalization.
No training, no new parameters, one extra scalar per retained tail
token. The certainty layer consists of the protected positions
(four attention sinks, thirty-two most recent tokens) and every token
whose allocation in \eqref{eq:pi} hits the upper clip;
$\varepsilon = 10^{-6}$ throughout.

\begin{assumption}[Known design]\label{as:design}
The probabilities $\pi_i$ in \eqref{eq:pi} are chosen by the algorithm,
stored, and bounded below by $\varepsilon$ on the tail.
\end{assumption}

\begin{assumption}[Bounded weights]\label{as:bounded}
Normalized weights are bounded: $a_i/(\pi_i \sum_{j\in\cC} a_j) \le B/m$
for all $i \in \cT$, guaranteed constructively by the floor
$\varepsilon$ and the certainty layer, which absorbs the head of the
score distribution.
\end{assumption}

Nothing is assumed about the quality of the score: a bad score
concentrates $\pi$ on the wrong tokens, which inflates the variance and
widens the certificate, but cannot bias the coverage, because validity
rests on Assumption~\ref{as:design} alone.

\begin{theorem}[Identifiability under Poisson design]
\label{thm:identifiability}
Write $N = \sum_{i \in \cC} a_i$ and let
$e_t^{\mathrm{lin}} = \tfrac{1}{N}\sum_{i \in \cT}
(\tfrac{I_i}{\pi_i} - 1)\, a_i (v_i - y_t)$
denote the first-order (linearized) error of \eqref{eq:hajek}, where
$I_i$ is the retention indicator. Under
Assumptions~\ref{as:design}--\ref{as:bounded},
\[
\Vhat_t \;=\; \frac{1}{\Nhat^2} \sum_{i \in \cS \cap \cT}
\frac{1 - \pi_i}{\pi_i^2}\, a_i^2\, \|v_i - \yhat_t\|^2,
\qquad \Nhat = \sum_{i \in \cS} a_i/\pi_i,
\]
satisfies $\E[\Vhat_t] = \Var(e_t^{\mathrm{lin}}) + O(B^2/m^2)$, and
$\Vhat_t$ is computable from the retained set alone. Under Poisson
sampling the Sen--Yates--Grundy double sum collapses to the single sum
above, so the cost is $O(|\cS \cap \cT|)$ per head per step.
\end{theorem}

The proof (Appendix~\ref{app:proofs}) is the classical HT variance
identity plus a Taylor expansion of the ratio; the $O(B^2/m^2)$ term is
the price of plugging $\yhat_t$ and $\Nhat$ into the unknown $y_t$ and
$N$. The theorem concerns the linearized error of the per-layer
attention output; Sections~\ref{sec:synthetic}
to~\ref{sec:vignettes} measure its task-level meaning.

\subsection{A certificate with empirically validated coverage}
\label{sec:cert}

Per-step variance estimates become a running error certificate through
the empirical-Bernstein construction. For step $t$ define the radius
\begin{equation}
r_t \;=\; \frac{\sqrt{2 \Vhat_t \log(1/\delta)} \;+\; b_t
\log(1/\delta)}{\|\yhat_t\| + \epsilon_0},
\qquad
b_t = \frac{\max_{i \in \cS\cap\cT}
\sqrt{(1-\pi_i)/\pi_i^2}\; a_i \|v_i - \yhat_t\|}{\Nhat},
\label{eq:radius}
\end{equation}
the empirical-Bernstein bound shape \citep{howard2021, waudbysmith2024}
instantiated with the design variance estimator of
Theorem~\ref{thm:identifiability}, a variance term plus a range term,
both computable from the retained set, targeting the relative
attention-output error at per-step confidence $1 - \delta$. The
certificate deployed in every experiment is this radius at
$\delta = 0.1$, averaged over a head and layer subsample and maximized
over the first six decode steps; its finite-sample coverage is measured
directly (Section~\ref{sec:replay}). Confidence level and probed
steps are design constants fixed once for all experiments: they scale
the radius and the operating point of the label-free rule, not its
failure ranking (Table~\ref{tab:certabl}). The anytime extension
(Remark~\ref{rem:anytime}, Appendix~\ref{app:proofs}) composes an
empirical-Bernstein e-process along decode steps with Ville's
inequality and e-value averaging across heads and layers
\citep{waudbysmith2024, ville1939, howard2021, vovk2021}.

\section{The estimator works where it is defined}
\label{sec:replay}

The retained-set variance estimator covers the true attention error,
ranks it, and costs no accuracy, on thirteen architectures. We replay
prefills of Qwen2.5-1.5B \citep{qwen2025} offline, evict at budgets
$\{12.5\%, 25\%, 50\%\}$, and compare against the full-cache output at
12{,}096 (layer, head, query) cells. At $\delta = 0.1$ the certificate
covers the realized error in 96.9\%, 97.2\%, and 97.7\% of cells at the three budgets, and its Spearman correlation with the true error is 0.943, 0.965, and 0.979. The median
certificate is roughly three times the median realized error at the
25\% budget (0.0997 against 0.0317): loose as an absolute bound, strong
as a ranking signal. Median relative error at the 25\% budget is 0.0317
for Poisson-with-H\'ajek against 0.0447 for top-$k$ and 0.2386 for
uniform sampling, with the same ordering at every budget; half of the
gap is the H\'ajek correction, and ablating the logit offset forfeits
it. Theorem~\ref{thm:impossibility} is realized constructively in the
same model: six random permutations of the evicted values move the true
error of the top-$k$ output from 0.014 to 0.898, a factor of 64, while
every retained-set statistic of top-$k$ is bit-identical across all six
worlds, and the certificate covers the realized error in all of them.
The same replay on twelve further architectures, spanning per-head and
full-width QK-norm, fused qkv projections, logit soft-capping and
mixture-of-experts layers, gives coverage 0.981--0.997 and Spearman
0.973--0.988, with Poisson+H\'ajek below top-$k$ on every model and
Poisson without the offset above it (Table~\ref{tab:zoo}).

\section{From attention error to task failure: synthetic suites}
\label{sec:synthetic}

The certificate transfers from attention error to task failure in 16 of
16 model--task cells. On passcode retrieval with generation-time
eviction (Qwen2.5-1.5B, 252 runs), with the question available before
compression (\emph{aware}) H2O succeeds fully and Poisson at 88.5\%,
and with compression before the question arrives (\emph{stream}) both
collapse to zero at these budgets. On its own failures H2O's retained
entropy scores AUC 0.405 (chance, $n = 48$) while the certificate scores
0.812 ($n = 192$; $z = 10.2$). Across four base models (Qwen2.5-1.5B/7B,
Llama-3.1-8B \citep{llama2024}, Mistral-7B-v0.3 \citep{jiang2023mistral})
and four RULER-style tasks \citep{hsieh2024ruler} at 640 scored
generations per cell, the certificate--failure AUC is positive in 16 of
16 cells, each 95\% interval excluding 0.5, mean 0.836, range
0.65--0.97, and thresholding it gives a risk--coverage knob: error 0.229
at 30\% coverage against 0.575 at full coverage. The exploration tax
concentrates in multi-needle retrieval at the 25\% budget, where
spreading $\pi$ over twenty needles costs 39--44 points on three of four
models, and is near zero elsewhere; all importance-based methods, ours
included, collapse in the stream condition at these budgets. Budget
allocation and compensation do not change this: PyramidKV
\citep{cai2024pyramidkv} and a CaM-style merge \citep{zhang2024cam}
average 0.050 and 0.056 against 0.047 for SnapKV at the 25\% budget on
four instruct models and four tasks, neither can flag its own failures,
and the gated system answers 0.434; at 50\% the deterministic arms reach
0.21--0.22 while the flag goes dark and the exploration tax leaves the
gated system at 0.125 (Table~\ref{tab:base}).

\section{Pre-registered study on real workloads, at two scales}
\label{sec:real}

\subsection{Design}

Four instruction-tuned models (Qwen2.5-1.5B/7B-Instruct,
Llama-3.1-8B-Instruct, Mistral-7B-Instruct-v0.3) run LongBench tasks
\citep{bai2024longbench} under a full grid: every example is answered by
the full cache, by StreamingLLM-, H2O-, and SnapKV-style deterministic
eviction, by question-aware top-$k$, and by Poisson eviction with the
online certificate (two seeds, stream and aware conditions), at budgets
$\{25\%, 50\%\}$ with 6k-token contexts (four LongBench tasks plus a synthetic needle anchor; Appendix~\ref{app:details})
and $\{12.5\%, 25\%, 50\%\}$ with 16k-token contexts (HotpotQA,
PassageRetrieval-en, MuSiQue, NarrativeQA), roughly 74{,}000
generations in total. Every compressed run logs the online self-signals
of Section~\ref{sec:panel} (Appendix~\ref{app:details}) and, for
Poisson arms, the certificate. Success is token-F1 $\ge 0.5$ for QA and
exact match for retrieval and needle tasks; \emph{eviction-induced}
failure means the full-cache run answers correctly and the compressed
run does not. Four claims with kill conditions (R1--R4) were fixed
before any full shard ran and three scale hypotheses (S1--S3) before
any 16k shard; Appendix~\ref{app:prereg} reproduces the registration
and every verdict (Table~\ref{tab:verdicts}).

\subsection{Where eviction damages answers}
\label{sec:whydied}

Question-aware eviction is nearly free on real tasks, and the damage
regime lives in streaming. At 16k, full-cache accuracy against Poisson
at the 25\% budget is 0.515 versus 0.516 on HotpotQA and 0.820 versus
0.815 on PassageRetrieval (Table~\ref{tab:acc16k}); the eviction-induced
failure rate in the aware condition, pooled over the four 16k tasks, is
6.2\% at the 25\% budget and 10.5\% at the harshest 12.5\% budget.
Neither truncation (S1: 4.0\% at 16k against 5.6\% at 6k on shared
tasks) nor memorization (context-dependent examples only: about one
point; Appendix~\ref{app:details}) manufactures the finding.

\begin{figure}[t]
\centering
\includegraphics[width=\linewidth]{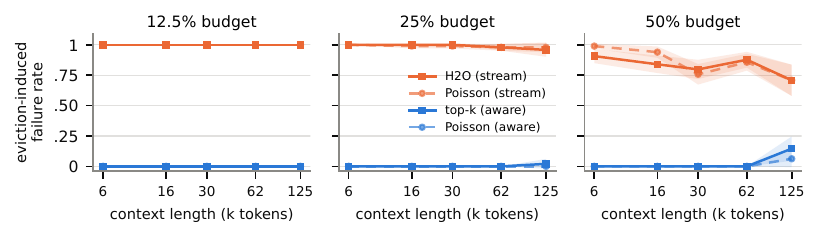}
\caption{Question-aware eviction is nearly free and streaming eviction
is catastrophic at every context length from 6k to 125k tokens.
Eviction-induced failure rate (among examples the full cache answers
correctly) on exact-length synthetic retrieval, three budgets, YaRN
beyond 32k; bands are 95\% binomial intervals.}
\label{fig:damage}
\end{figure}

Streaming eviction, where tokens are deleted before the question
arrives, breaks 40--58\% of answerable examples on the same tasks. On
exact-length synthetic retrieval it destroys 70--100\% at every context
length from 6k to 125k tokens while the aware condition stays near zero
across the same sweep (Figure~\ref{fig:damage}); the full cache answers
97--100\% up to 62k and 80\% at 125k, the price of YaRN extrapolation,
which conditioning on full-correct examples removes. Real-task pooling
at natural lengths near 23k shows the same split, streaming 35--54\%
against aware 2--17\%.

Two further pre-registered outcomes fix the division of labor. The
randomized design costs nothing on real tasks: the aware-condition gap
between deterministic top-$k$ and Poisson is at most 0.7 points at 6k
and 0.3 points at 16k, at every budget including 12.5\% (S3). And mean
output log-probability predicts failure better than every cache-side
signal, certificate included: overall failure 0.73--0.80 against
0.56--0.57 (LongBench-pooled), induced failure 0.782 against 0.778 at
6k and 0.806 against 0.768 at 16k (paired difference $-0.038$, 95\% CI
$[-0.065, -0.010]$), and risk--coverage. The certificate's own pooled
failure-prediction AUC on LongBench is 0.555 at 6k and 0.572 at 16k
(R1, S2; Table~\ref{tab:permodel}) and 0.500 in the aware condition,
because there is almost no eviction damage to see: failures on these
tasks are overwhelmingly inherent. Selective answering should be gated
on output confidence \citep{hendrycks2017, geifman2017}; a trust signal
evaluated without this baseline overstates its case.

\subsection{The silent-failure panel}
\label{sec:panel}

No deterministic self-signal sees its own eviction-induced failures,
and the certificate leads the panel by ten points.
Table~\ref{tab:panelmain} scores every self-signal available to a
deterministic evictor on predicting that evictor's own induced
failures, pooled over both scales (Table~\ref{tab:panel} gives
6k-suite values with intervals). Retained-attention entropy sits at 0.43--0.51 for
StreamingLLM, H2O, and SnapKV, and keep-boundary margin at 0.49--0.50
everywhere. Evicted score mass is the one exception, with partial signal
(0.59--0.73 across arms and suites) that still trails the certificate by
ten points and washes out for overall failure prediction (0.59 pooled),
where task difficulty dominates.

\begin{table}[t]
\centering
\small
\setlength{\tabcolsep}{4pt}
\resizebox{\linewidth}{!}{\begin{tabular}{@{}lccccc@{}}
\toprule
self-signal & StreamingLLM & H2O & SnapKV & aware top-$k$ & Poisson (ours) \\
\midrule
retained entropy & 0.43 [0.40,0.46] & 0.47 [0.44,0.50] & 0.50 [0.47,0.53] & 0.62 [0.57,0.68] & 0.48 [0.45,0.50] \\
evicted score mass & -- & 0.67 [0.65,0.68] & 0.67 [0.65,0.69] & 0.59 [0.55,0.63] & -- \\
keep-boundary margin & -- & 0.49 [0.49,0.50] & 0.50 [0.49,0.50] & 0.50 [0.49,0.50] & -- \\
certificate & -- & -- & -- & -- & \textbf{0.77 [0.76,0.78]} \\
\bottomrule
\end{tabular}}
\caption{Deterministic self-signals sit at chance on their own eviction-induced failures; the certificate, which only a known randomized design provides, leads the panel by ten points. AUC for predicting the arm's own eviction-induced failures, pooled over the 6k and 16k LongBench suites (pre-registered runs), with 95\% cluster-bootstrap intervals resampling examples; evicted score mass is the one deterministic signal with partial information.}
\label{tab:panelmain}
\end{table}

\subsection{What survives: attribution, and scheduling}
\label{sec:attribution}

The certificate loses at predicting whether an answer is wrong and wins
at a question output confidence cannot pose: whose fault was it?
Restricting to failures and asking each signal to separate
eviction-induced from inherent ones, the certificate scores AUC 0.749 at
6k ($n = 7{,}924$ failures, 1{,}133 induced), 0.727 at 16k
($n = 11{,}717$, 1{,}776 induced), and 0.654 at the natural 23k-token
lengths of musique and narrativeqa ($n = 918$, 95 induced); output
log-probability scores 0.542, 0.469, and 0.519
(Figure~\ref{fig:attribution}). The decline with length is
mechanical: a fixed ratio leaves more resident tokens as context grows,
induced failures become rare (14\% of failures at 6k, 10\% at 23k)
and borderline, and 95 of them at 23k leave a wide interval. The asymmetry is structural: low
confidence flags hard examples whether or not the cache is at fault,
while the certificate reads the sampling noise of the retained set and
responds only to the compression channel. This is
Theorem~\ref{thm:identifiability} at task level: the design identifies
the error of the channel the design controls, and nothing else. Triage
against full-cache reruns confirms the verdicts: the tercile the
certificate blames most recovers 2.4$\times$ as often as the tercile it
exonerates (Appendix~\ref{app:details}).

\begin{figure}[t]
\centering
\includegraphics[width=0.78\linewidth]{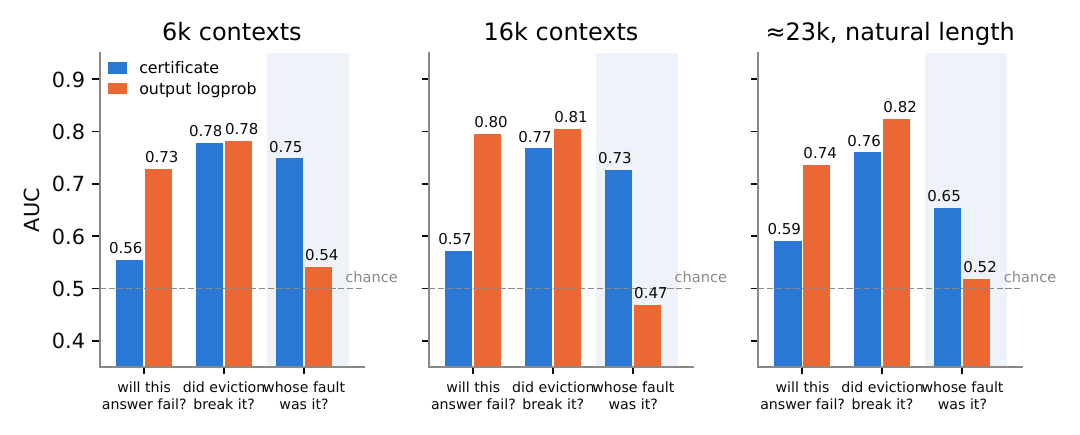}
\caption{The certificate wins attribution at every scale while output
confidence wins prediction. Three questions at 6k, 16k, and natural
lengths near 23k tokens (musique and narrativeqa); the shaded column is
attribution.}
\label{fig:attribution}
\end{figure}

A certificate-free attributor (two independent Poisson draws per
query) reaches AUC 0.620--0.687 on label-assisted proxies, at double
the cost and with no validity statement (Table~\ref{tab:dprime}).

Attribution converts into a scheduling product: under streaming
compression, a quarter of queries may be re-run with the full cache.
Certificate gating captures 36\% (31\%) of the oracle's gain at the
12.5\% (25\%) budget, 1.7--1.8 times the gain of random gating and up
to 2.1 times that of confidence gating (Figure~\ref{fig:recompute},
Table~\ref{tab:dprime}). At the 12.5\% budget (16k) certificate gating
recovers 0.33 against 0.29 for random and 0.28--0.29 for confidence
gating, $+3.7$ points over random per seed (95\% interval
$[+2.3, +5.0]$) and $+4.4$ to $+4.9$ over confidence; the ordering
holds at 25\% ($+1.9$ to $+2.2$ over random) and at 6k ($+2.7$ to
$+3.0$), intervals excluding zero. Confidence gating lands below
random at the harsh budget because it spends re-runs on inherently hard
examples; the certificate spends them on cache-damaged ones, which
recomputation can save. At equal accuracy, the certificate matches
random gating at a quarter of the queries by re-running 10.8--15.6\% of
them; log-probability needs a per-task threshold to do as well
(Table~\ref{tab:iso}).

\section{Deployment: certificate-gated agent memory}
\label{sec:vignettes}

One quantity organizes this section: the resident cache in tokens, not
the budget ratio. The rule $\mathrm{cert} \ge 1$ is the same
everywhere and fires when the resident cache falls to a few hundred
tokens (225--308 at ten turns, 340--460 at thirty, 400--1{,}000 on real
conversations); there the gated system answers where top-$k$ fails.

\paragraph{Agent memory rots silently; the certificate is the smoke
alarm.} Fifty ten-turn dialogues embed one personal fact per user turn
in natural chatter; after prefill (about 1{,}060 tokens) the cache is
compressed to a 25\% budget in the stream condition and four recall
questions of varying age follow (Qwen2.5-7B-Instruct and
Llama-3.1-8B-Instruct, 2{,}000 scored answers; Figure~\ref{fig:agent}).
The full cache answers 98.5\%, so almost every compressed failure is eviction-induced. At the 25\% budget
facts one or two turns old survive at 46--61\% under H2O and Poisson
and facts three or more turns old at 17\% or less under every
importance-based arm. On each arm's own failures the certificate reaches
AUC 0.84 (Qwen) and 0.79 (Llama) against 0.72 and 0.66 for H2O's
retained entropy, and where every entropy operating point needs
labeled answers, the label-free rule $\mathrm{cert} \ge \tau = 1$
flags 99\% of failures at a 49\% false-alarm rate at the 15\% budget
on both models, and 54\% and 28\% at 25\%, where the certificate
median sits at the threshold (0.99 on Qwen, 0.91 on Llama).

\paragraph{Thirty turns: the boundary tracks damage, not the ratio.}
Tripling the dialogue to thirty turns and thirty distinct facts (3,160
tokens; fifty dialogues; six recall questions each) moves the regime
boundary to a smaller budget ratio. At a 25\% budget (about 790 resident
tokens) the median certificate reads 0.76 and the red flag correctly
stays dark while both eviction rules collapse to 0.13--0.23 accuracy:
retrieval-type amnesia, which a per-step certificate does not measure.
The heavy-damage regime returns below 15\%: at 12\% the median
certificate is 1.09 and the red-flag rate 86\%, and at 10\% (about 320
resident tokens) 1.18 and 95\%, where the gated system recovers 0.97
against 0.09 for top-$k$ (AUC 0.82) and, for facts stated 26 to 30
turns earlier, 97\% against 2\%. A ratio rule would put the boundary at
790 resident tokens; it sits at 270 for ten turns and near 400 for
thirty.

\begin{figure}[t]
\centering
\includegraphics[width=0.72\linewidth]{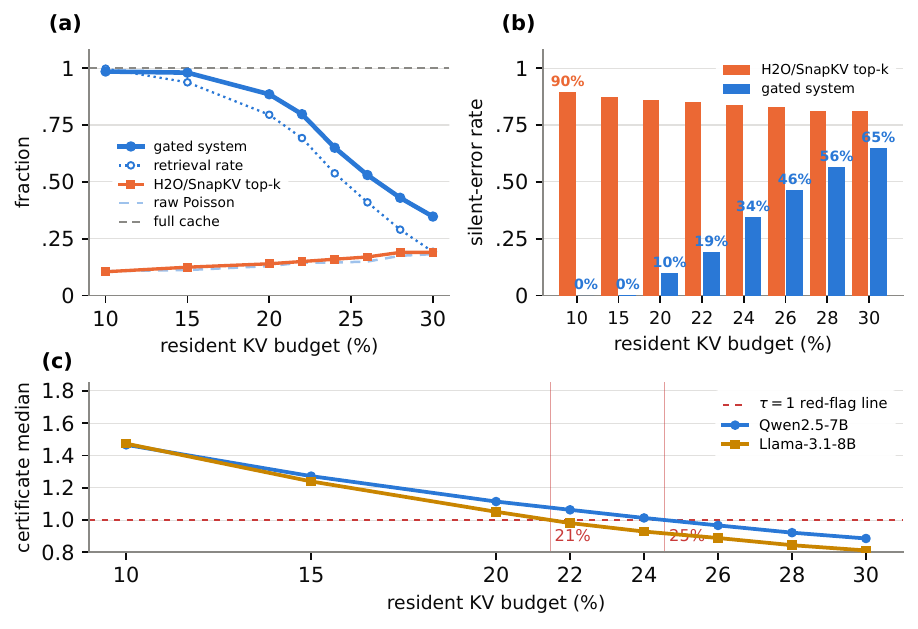}
\caption{A larger budget buys the fraction of answers that need no
retrieval, not accuracy, and the handover happens where the median
certificate crosses $\tau=1$. Merged 10--30\% budget grid (Qwen2.5-7B,
fifty dialogues, stream condition): (a) gated accuracy and retrieval
rate, with raw Poisson and top-$k$; (b) silent-error rate; (c)
certificate medians for both models against $\tau = 1$.}
\label{fig:budget}
\end{figure}

\paragraph{The budget buys trust, not accuracy.}
Sweeping the budget across the 10--30\% grid (same fifty dialogues;
Figure~\ref{fig:budget}), panel (a): gated accuracy falls as the budget
grows, from 98.5\% at a 10\% budget to 35\% at 30\%, while raw Poisson
and top-$k$ stay pinned at 11--19\%. A larger budget lowers reconstruction damage, the
certificate hands trust back to the compressed cache, and the retrieval
rate slides from 100\% to 19\%, so accuracy follows it down. The budget
buys the fraction of answers that need no retrieval, not accuracy.
Panel (b): the raw failure rate barely moves (89\% to 82\%) while
the silent fraction of delivered errors rises from zero to 80\%, as
red-flag coverage of failures falls from 100\% to 20\% and the failure
population's median certificate slides from 1.48 to 0.90. At tight budgets failures are
reconstruction distortion, which the certificate measures; at looser
budgets they are retrieval-type amnesia, invisible to a per-step
certificate by construction, while the ranking signal itself holds
(failure AUC 0.92 to 0.80). Llama-3.1-8B reproduces every pattern
(Appendix~\ref{app:details}), and H2O ships every error silently at
every budget. Panel (c) locates the handover: the median certificate
declines almost linearly with budget and crosses $\tau=1$ at 25\% on
Qwen and 21\% on Llama, where the red-flag rate passes one half: the
exit from the heavy-damage regime is a design quantity, not a tuned
hyperparameter.

\paragraph{Real conversations.}
On LoCoMo \citep{maharana2024locomo}, ten real conversations of
11k--25k tokens, 699 factual questions per model, the full cache scores
F1 0.297 (Qwen2.5-7B), 0.327 (Llama-3.1-8B) and 0.397 (Qwen2.5-32B). The
heavy-damage regime sits at a resident cache of roughly 400--1{,}000
tokens on every model, 2--3\% of the history for Llama and 3--5\% for
Qwen: there top-$k$ keeps 0.06--0.11 and the gated system returns the
full-cache score (0.297 on 7B, 0.316 on Llama, 0.390 on 32B). Above
that point the flag goes dark and the residual failures are inherent or
retrieval-type; at every budget the certificate beats output confidence
at attribution (Table~\ref{tab:locomo}).

\paragraph{Five families and a 32B replication.}
The rule $\mathrm{cert} \ge 1$ is the same on every model; the budget at
which it fires is model-specific because that budget is a measurement,
not a parameter. Across five families the ten-turn crossing sits at 21--29\% of
the budget but within 225--308 resident tokens, and at 10\% every gated
system answers 0.98--0.99 against 0.10--0.12 for top-$k$
(Table~\ref{tab:families}); a different operating point comes from the
label-free recomputation rate (Section~\ref{sec:attribution}), not from
$\tau$. The 32B replication keeps the certificate's ranking within 0.10 of the 7B model (AUC 0.80/0.74/0.70 against 0.87/0.84/0.80 at 20/25/30\%, Table~\ref{tab:scale}) and beats the 7B system at every budget (0.58 against 0.35 at 30\%).

\paragraph{Finer grains.}
The best single unit of the layer--head matrix (layer 22, one sampled head)
predicts failure at AUC 0.967, above the 0.862 of the uniform
aggregate. Ranking evicted 64-token blocks by the variance contribution
of their retained neighbors places the block holding the queried fact
first in 64\% of red-flagged failures (top three in 79\%, of
seventeen), and retrieving 10\% of the context by that ranking
recovers 0.70 accuracy against 0.18 for random blocks (oracle 0.89,
full retrieval 1.00; Table~\ref{tab:finegrain}): the red flag becomes
a targeted retrieval.

\section{Discussion}
\label{sec:discussion}

\paragraph{Guidance by regime.}
Where the query is known before compression, eviction at 25--50\%
budgets is close to free; for answer-level trust, gate on output
confidence. In streaming and agent-memory settings the certificate is the
strongest online attributor we measured and the only one whose validity
is tied to the design; damage monitoring and recomputation scheduling
are where that quantity is worth money.

\paragraph{Limitations and future directions.}
Contexts reach 125k tokens, models 32B parameters and real tasks 23k; beyond these the trends are untested. The certificate concerns per-layer attention error, its task-level meaning is empirical, and the anytime extension has two open ends (Appendix~\ref{app:proofs}). Budget ratios do not transfer across families (Table~\ref{tab:locomo}), the operating point follows $\delta$ and the probed steps (Table~\ref{tab:certabl}), $\epsilon_0$ is unablated and fixed-size designs \citep{hajek1964} are tested only in replay. The Python-level probe adds a fixed 0.2--0.3 s per answer (15--16\% of a 128-token answer) and the unfused offset 12\% per token; a fused kernel is next.

\section*{Reproducibility statement}
All experiments run on single H100 or H200 GPUs (roughly 150 GPU-hours
total). Per-run JSON logs (about 100{,}000 scored generations across
the task suites, controls, and vignettes, plus 12{,}096 replay cells),
the exact prompts, the analysis scripts that regenerate every number
and figure from those logs, and the timestamped pre-registration files
are packaged for release. Models and datasets are public (Qwen2.5,
Llama-3.1, Mistral-7B; LongBench, RULER-style generators).
Appendix~\ref{app:details} gives the harness, signal definitions, task
sampling, scoring, and statistics; Appendix~\ref{app:prereg}
reproduces the pre-registration record.

\bibliographystyle{plainnat}
\bibliography{refs}

@inproceedings{zhang2023h2o,
  title={{H$_2$O}: Heavy-hitter oracle for efficient generative inference of large language models},
  author={Zhang, Zhenyu and Sheng, Ying and Zhou, Tianyi and Chen, Tianlong and Zheng, Lianmin and Cai, Ruisi and Song, Zhao and Tian, Yuandong and R{\'e}, Christopher and Barrett, Clark and Wang, Zhangyang and Chen, Beidi},
  booktitle={Advances in Neural Information Processing Systems 36},
  year={2023},
  note={arXiv:2306.14048}
}

@inproceedings{xiao2024streamingllm,
  title={Efficient streaming language models with attention sinks},
  author={Xiao, Guangxuan and Tian, Yuandong and Chen, Beidi and Han, Song and Lewis, Mike},
  booktitle={International Conference on Learning Representations},
  year={2024},
  note={arXiv:2309.17453}
}

@inproceedings{li2024snapkv,
  title={{SnapKV}: {LLM} knows what you are looking for before generation},
  author={Li, Yuhong and Huang, Yingbing and Yang, Bowen and Venkitesh, Bharat and Locatelli, Acyr and Ye, Hanchen and Cai, Tianle and Lewis, Patrick and Chen, Deming},
  booktitle={Advances in Neural Information Processing Systems 37},
  year={2024},
  note={arXiv:2404.14469}
}

@inproceedings{oren2024tova,
  title={Transformers are multi-state {RNNs}},
  author={Oren, Matanel and Hassid, Michael and Yarden, Nir and Adi, Yossi and Schwartz, Roy},
  booktitle={Proceedings of EMNLP},
  year={2024},
  note={arXiv:2401.06104}
}

@article{cai2024pyramidkv,
  title={{PyramidKV}: Dynamic {KV} cache compression based on pyramidal information funneling},
  author={Cai, Zefan and Zhang, Yichi and Gao, Bofei and Liu, Yuliang and Liu, Tianyu and Lu, Keming and Xiong, Wayne and Dong, Yue and Chang, Baobao and Hu, Junjie and Xiao, Wen},
  journal={arXiv preprint arXiv:2406.02069},
  year={2024}
}

@article{feng2024adakv,
  title={{Ada-KV}: Optimizing {KV} cache eviction by adaptive budget allocation for efficient {LLM} inference},
  author={Feng, Yuan and Lv, Junlin and Cao, Yukun and Xie, Xike and Zhou, S. Kevin},
  journal={arXiv preprint arXiv:2407.11550},
  year={2024}
}

@inproceedings{zhang2024cam,
  title={{CaM}: Cache merging for memory-efficient {LLMs} inference},
  author={Zhang, Yuxin and Du, Yuxuan and Luo, Gen and Zhong, Yunshan and Zhang, Zhenyu and Liu, Shiwei and Ji, Rongrong},
  booktitle={International Conference on Machine Learning},
  year={2024}
}

@article{momentkv2026,
  title={{MomentKV}: Closing the directional gap in {KV} cache eviction for long-context inference},
  author={Li, Yu and Li, Binxu and Lan, Tian},
  journal={arXiv preprint arXiv:2606.01563},
  year={2026}
}

@inproceedings{chen2025magicpig,
  title={{MagicPIG}: {LSH} sampling for efficient {LLM} generation},
  author={Chen, Zhuoming and Sadhukhan, Ranajoy and Ye, Zihao and Zhou, Yang and Zhang, Jianyu and Nolte, Niklas and Tian, Yuandong and Douze, Matthijs and Bottou, L{\'e}on and Jia, Zhihao and Chen, Beidi},
  booktitle={International Conference on Learning Representations},
  year={2025},
  note={arXiv:2410.16179}
}

@article{desai2025vattention,
  title={{vAttention}: Verified sparse attention},
  author={Desai, Aditya and Agrawal, Kumar Krishna and others},
  journal={arXiv preprint arXiv:2510.05688},
  year={2025}
}

@article{nexus2026,
  title={Forget without compromise: Nexus sampling for streaming {KV}-cache eviction under fixed budgets},
  author={Duong, Duc and Le, Hoang Anh Duy and Xie, Jianwen and Shrivastava, Anshumali and Xu, Zhaozhuo},
  journal={arXiv preprint arXiv:2606.23961},
  year={2026}
}

@article{vase2026,
  title={Value-aware stochastic {KV} cache eviction for reasoning models},
  author={Chang, Ting-Yun and Fu, Harvey Yiyun and Fu, Deqing and Yang, Chenghao and Thomason, Jesse and Jia, Robin},
  journal={arXiv preprint arXiv:2606.03928},
  year={2026}
}

@article{fixedcontract2026,
  title={When does value-aware {KV} eviction help? {A} fixed-contract diagnostic for non-monotone cache compression},
  author={Zhang, Ruijie and Liang, Haozhe and Chang, Da and Hu, Li and Kong, Fanqi and Yin, Huaxiao and Li, Yu},
  journal={arXiv preprint arXiv:2605.08234},
  year={2026}
}

@article{goel2025caote,
  title={{CAOTE}: {KV} cache selection for {LLMs} via attention output error-based token eviction},
  author={Goel, Raghavv and Park, Junyoung and Gagrani, Mukul and Jones, Dalton and Morse, Matthew and Langston, Harper and Lee, Mingu and Lott, Chris},
  journal={arXiv preprint arXiv:2504.14051},
  year={2025}
}

@inproceedings{tang2024quest,
  title={Quest: Query-aware sparsity for efficient long-context {LLM} inference},
  author={Tang, Jiaming and Zhao, Yilong and Zhu, Kan and Xiao, Guangxuan and Kasikci, Baris and Han, Song},
  booktitle={International Conference on Machine Learning},
  year={2024},
  note={arXiv:2406.10774}
}

@article{horvitz1952,
  title={A generalization of sampling without replacement from a finite universe},
  author={Horvitz, Daniel G. and Thompson, Donovan J.},
  journal={Journal of the American Statistical Association},
  volume={47},
  number={260},
  pages={663--685},
  year={1952}
}

@article{yates1953,
  title={Selection without replacement from within strata with probability proportional to size},
  author={Yates, Frank and Grundy, P. Michael},
  journal={Journal of the Royal Statistical Society: Series B},
  volume={15},
  number={2},
  pages={253--261},
  year={1953}
}

@article{sen1953,
  title={On the estimate of the variance in sampling with varying probabilities},
  author={Sen, Amode R.},
  journal={Journal of the Indian Society of Agricultural Statistics},
  volume={5},
  pages={119--127},
  year={1953}
}

@article{hajek1964,
  title={Asymptotic theory of rejective sampling with varying probabilities from a finite population},
  author={H{\'a}jek, Jaroslav},
  journal={Annals of Mathematical Statistics},
  volume={35},
  number={4},
  pages={1491--1523},
  year={1964}
}

@book{sarndal1992,
  title={Model Assisted Survey Sampling},
  author={S{\"a}rndal, Carl-Erik and Swensson, Bengt and Wretman, Jan},
  publisher={Springer},
  year={1992}
}

@article{neyman1934,
  title={On the two different aspects of the representative method},
  author={Neyman, Jerzy},
  journal={Journal of the Royal Statistical Society},
  volume={97},
  number={4},
  pages={558--625},
  year={1934}
}

@book{kish1965,
  title={Survey Sampling},
  author={Kish, Leslie},
  publisher={Wiley},
  year={1965}
}

@book{ville1939,
  title={{\'E}tude critique de la notion de collectif},
  author={Ville, Jean},
  publisher={Gauthier-Villars},
  year={1939}
}

@article{howard2021,
  title={Time-uniform, nonparametric, nonasymptotic confidence sequences},
  author={Howard, Steven R. and Ramdas, Aaditya and McAuliffe, Jon and Sekhon, Jasjeet},
  journal={Annals of Statistics},
  volume={49},
  number={2},
  pages={1055--1080},
  year={2021}
}

@article{waudbysmith2024,
  title={Estimating means of bounded random variables by betting},
  author={Waudby-Smith, Ian and Ramdas, Aaditya},
  journal={Journal of the Royal Statistical Society: Series B},
  volume={86},
  number={1},
  pages={1--27},
  year={2024}
}

@article{vovk2021,
  title={E-values: Calibration, combination and applications},
  author={Vovk, Vladimir and Wang, Ruodu},
  journal={Annals of Statistics},
  volume={49},
  number={3},
  pages={1736--1754},
  year={2021}
}

@article{ramdas2023,
  title={Game-theoretic statistics and safe anytime-valid inference},
  author={Ramdas, Aaditya and Gr{\"u}nwald, Peter and Vovk, Vladimir and Shafer, Glenn},
  journal={Statistical Science},
  volume={38},
  number={4},
  pages={576--601},
  year={2023}
}

@inproceedings{hendrycks2017,
  title={A baseline for detecting misclassified and out-of-distribution examples in neural networks},
  author={Hendrycks, Dan and Gimpel, Kevin},
  booktitle={International Conference on Learning Representations},
  year={2017},
  note={arXiv:1610.02136}
}

@inproceedings{geifman2017,
  title={Selective classification for deep neural networks},
  author={Geifman, Yonatan and El-Yaniv, Ran},
  booktitle={Advances in Neural Information Processing Systems 30},
  year={2017}
}

@inproceedings{bai2024longbench,
  title={{LongBench}: A bilingual, multitask benchmark for long context understanding},
  author={Bai, Yushi and Lv, Xin and Zhang, Jiajie and Lyu, Hongchang and Tang, Jiankai and Huang, Zhidian and Du, Zhengxiao and Liu, Xiao and Zeng, Aohan and Hou, Lei and Dong, Yuxiao and Tang, Jie and Li, Juanzi},
  booktitle={Proceedings of ACL},
  year={2024},
  note={arXiv:2308.14508}
}

@inproceedings{hsieh2024ruler,
  title={{RULER}: What's the real context size of your long-context language models?},
  author={Hsieh, Cheng-Ping and Sun, Simeng and Kriman, Samuel and Acharya, Shantanu and Rekesh, Dima and Jia, Fei and Zhang, Yang and Ginsburg, Boris},
  booktitle={First Conference on Language Modeling},
  year={2024},
  note={arXiv:2404.06654}
}

@article{feld1991,
  title={Why your friends have more friends than you do},
  author={Feld, Scott L.},
  journal={American Journal of Sociology},
  volume={96},
  number={6},
  pages={1464--1477},
  year={1991}
}

@article{eom2014,
  title={Generalized friendship paradox in complex networks: The case of scientific collaboration},
  author={Eom, Young-Ho and Jo, Hang-Hyun},
  journal={Scientific Reports},
  volume={4},
  pages={4603},
  year={2014}
}

@article{qwen2025,
  title={Qwen2.5 technical report},
  author={{Qwen Team}},
  journal={arXiv preprint arXiv:2412.15115},
  year={2024}
}

@article{llama2024,
  title={The {Llama} 3 herd of models},
  author={Grattafiori, Aaron and others},
  journal={arXiv preprint arXiv:2407.21783},
  year={2024}
}

@article{jiang2023mistral,
  title={Mistral {7B}},
  author={Jiang, Albert Q. and Sablayrolles, Alexandre and Mensch, Arthur and Bamford, Chris and Chaplot, Devendra Singh and de las Casas, Diego and Bressand, Florian and Lengyel, Gianna and Lample, Guillaume and Saulnier, Lucile and others},
  journal={arXiv preprint arXiv:2310.06825},
  year={2023}
}

@inproceedings{maharana2024locomo,
  title={Evaluating Very Long-Term Conversational Memory of {LLM} Agents},
  author={Maharana, Adyasha and Lee, Dong-Ho and Tulyakov, Sergey and Bansal, Mohit and Barbieri, Francesco and Fang, Yuwei},
  booktitle={Proceedings of ACL},
  year={2024},
  note={arXiv:2402.17753}
}

\clearpage
\appendix

\section{Proofs and constructions}
\label{app:proofs}

\subsection{Theorem~\ref{thm:identifiability}}
Condition on the query $q_t$ and the cache contents; randomness is only
the Poisson indicators $I_i$, $i \in \cT$, with known $\pi_i$.
Write $x_i = a_i (v_i - y_t)/N$ and note $\sum_{i \in \cC} x_i = 0$ by
\eqref{eq:full}, so the linearized error is
$e_t^{\mathrm{lin}} = \sum_{i \in \cT} (I_i/\pi_i - 1)\, x_i$.
Independence across $i$ gives
\[
\Var\!\left(e_t^{\mathrm{lin}}\right)
= \sum_{i \in \cT} \frac{1 - \pi_i}{\pi_i}\, \|x_i\|^2,
\]
the Poisson-sampling (single-sum) case of the Sen--Yates--Grundy form
\citep{sen1953, yates1953}: joint inclusion probabilities factor,
$\pi_{ij} = \pi_i \pi_j$, so the double sum vanishes. The
Horvitz--Thompson estimator of that population total,
\[
\sum_{i \in \cS \cap \cT} \frac{1}{\pi_i} \cdot
\frac{1 - \pi_i}{\pi_i}\, \|x_i\|^2
= \sum_{i \in \cS \cap \cT} \frac{1 - \pi_i}{\pi_i^2}\, \|x_i\|^2,
\]
is unbiased for it \citep{horvitz1952}. $\Vhat_t$ replaces the unknown
$x_i$ by $a_i (v_i - \yhat_t)/\Nhat$. A first-order expansion of the
H\'ajek ratio \citep{sarndal1992} gives
$\yhat_t - y_t = e_t^{\mathrm{lin}} + O_p(B^2/m^2)$ and
$\Nhat/N = 1 + O_p(B/m)$ under Assumption~\ref{as:bounded}, since each
summand contributes at most $B/m$ after normalization and the tail
sample has expected size $m$. Propagating both substitutions through the
quadratic form perturbs the expectation by $O(B^2/m^2)$, which is the
stated remainder. \hfill$\square$

\subsection{The anytime construction}

\begin{remark}[Anytime extension]
\label{rem:anytime}
The per-step radius \eqref{eq:radius} extends toward time-uniform
validity by a standard route: an empirical-Bernstein e-process along
decode steps within a head \citep{waudbysmith2024}, Ville's inequality
\citep{ville1939, howard2021} for the uniform-in-$T$ statement, and
arithmetic averaging of e-values across arbitrarily dependent heads,
layers, and steps, which preserves validity \citep{vovk2021}. Under
this construction, coverage does not depend on when the certificate is
read.
\end{remark}

Fix a head. For each $t$ the summands $\{(I_i/\pi_i - 1)\, x_i\}_{i \in
\cT}$ are bounded by Assumption~\ref{as:bounded} and mean-zero given
the past, so the empirical-Bernstein supermartingale of
\citet{waudbysmith2024} applied to the running sums yields a process
$M_t$ with $\E[M_t] \le 1$ under the null that the realized error stays
within \eqref{eq:radius}; Ville's inequality \citep{ville1939} converts
$M_t$ into the time-uniform statement, and the boundary in
\citep{howard2021} gives the stated radius shape (a variance term plus
a range term, both computable from the retained set). If tail
indicators are redrawn every $T_r$ steps the increments are independent
across blocks; if the same draw is reused, increments within a block
are identical and the product is taken over blocks, which only loosens
the bound. For combination: if $E^{(1)}, \dots, E^{(K)}$ are e-values
for the per-(layer, head) nulls, arbitrarily dependent, then $\bar{E} =
K^{-1}\sum_k E^{(k)}$ satisfies $\E[\bar E] \le 1$ under the
intersection null, so thresholding $\bar{E}$ at $1/\delta$ is valid
\citep{vovk2021}; the linearization remainder of
Theorem~\ref{thm:identifiability} adds the $O(B^2/m^2)$ slack. The two
open ends are the reduction from the vector-valued linearized error to
a scalar supermartingale, whose constants are not settled, and the
reuse of one tail draw across steps, which the block structure handles
conservatively. \hfill$\square$

\subsection{Two remarks}

\begin{remark}[One law, two architectures]
The bias of the uncorrected retained softmax is, to first order,
$\mathrm{Cov}_p(a, v)/\E_p[a]$ under the retention distribution: a
size-biased covariance term. The same algebra drives degree bias in
message-passing graph networks, where the friendship paradox
\citep{feld1991} makes high-degree neighbors over-represented, and the
Eom--Jo sign criterion \citep{eom2014} predicts opposite intervention
directions in the two systems: graph hubs are over-counted and need
down-weighting, while attention sinks drain value mass and need
retention. A companion manuscript develops the graph side; the law is
narrative context here and carries no load in the proofs.
\end{remark}

\begin{remark}[Where existing methods sit]
\label{rem:methods}
Every eviction scheme is a triple (design, estimator, variance
handling). H2O-style top-$k$ is a certainty-only design with a plug-in
estimator and no variance layer; MagicPIG is sampling with
self-normalized importance-sampling correction and no variance layer;
Nexus and VASE are sampling designs without correction; CaM, MomentKV,
and related merging methods are deterministic designs with a stratified
ratio compensation of the point estimate. The certificate layer of this
paper is orthogonal and can be attached to any known-probability
design. The stratified-ratio view suggests that merging error should
scale with the within-stratum dispersion of evicted values rather than
with evicted attention mass. We tested this in the replay framework
with a CaM-style merge, each evicted value folded into the retained
token with the nearest key, weighted by attention mass: the merging
error tracks evicted mass (Spearman 0.70--0.92 across four models and
three budgets) and not dispersion ($-0.07$ to $0.10$), because the merged
representative inherits the retained key's score, so the score mismatch
dominates the value dispersion. Merging removes under 8\% of top-$k$'s
attention error, the H\'ajek-corrected Poisson design 15--50\% at the
same budgets (Table~\ref{tab:replayv2}), and neither merge produces an
error estimate.
\end{remark}

\section{Experimental details}
\label{app:details}

\paragraph{Harness.}
All suites share one offline-replay harness. A prompt is prefilled once;
importance scores are computed from accumulated attention (H2O-style,
strided prefill queries), an observation window (SnapKV-style, last 64
prefill positions), or question-position queries (aware); each arm then
receives its own compressed \texttt{DynamicCache} and generates greedily
with per-step signal instrumentation. Protected positions (4 sinks, 32
recent) are never evicted. Poisson arms draw independent Bernoulli
retention from \eqref{eq:pi} with the mass of the certainty layer chosen
by the same budget accounting as the deterministic arms. The H\'ajek
offset $+\log(1/\pi_i)$ enters the model as an additive attention bias
on the retained keys. The pre-registered runs and the synthetic suites
were run with the uncorrected retained softmax and the certificate
computed as the H\'ajek radius. The Poisson arms of both LongBench
suites were then re-run with the offset on (four models, two budgets,
two seeds, stream and aware; 6{,}400 generations per suite). At 16k:
accuracy 0.239 and 0.340 at the 12.5\% and 25\% budgets against 0.243
and 0.344 uncorrected, certificate failure-prediction AUC 0.51--0.53
against 0.48--0.49 with log-probability at 0.75--0.76 (0.75--0.78
uncorrected), attribution AUC 0.65 against 0.65--0.67 (log-probability
0.38--0.39 against 0.40--0.42), and certificate gating at 32--33\% of
the oracle gain, 1.5--1.8 times random and 1.7--1.9 times confidence
gating, against 31--36\%, 1.7--1.8 and 1.8--2.1 uncorrected. At 6k
(four LongBench tasks, 25\% and 50\% budgets): accuracy 0.291 and 0.367
against 0.290 and 0.368, failure-prediction AUC 0.55--0.56 against
0.53 with log-probability at 0.71--0.73 in both, attribution AUC
0.65--0.74 against 0.64--0.72 (log-probability 0.49--0.50 against
0.49--0.52), and certificate gating $+3.3$ and $+2.9$ points over random
against $+2.9$ and $+3.0$. The offset lowers the certificate's scale by
4--8\% (median 0.72 against 0.75 at 6k, 0.85 against 0.92 at 16k), so
only quantities that read the scale move; every ranking-based
conclusion of Section~\ref{sec:real} holds in the certified
configuration, and a paired re-run of the synthetic suite (two instruct
models, four tasks, 28 cells) leaves Poisson accuracy at 0.356 against
0.356 with every cell within 0.03.
The agent-memory experiments of Section~\ref{sec:vignettes}, whose
red-flag rule depends on the certificate's scale, were run with the
offset on. The certificate is \eqref{eq:radius} at
$\delta = 0.1$ computed on every fourth head of every layer during the
first six decode steps, averaged, then maximized over steps. The
probability floor in \eqref{eq:pi} is $\varepsilon = 10^{-6}$ and the
normalizer constant in \eqref{eq:radius} is $\epsilon_0 = 10^{-9}$.
Poisson sampling leaves the retained-set size random. Fixed-size
systematic sampling at the same inclusion probabilities
\citep{hajek1964} keeps the budget exact; in the replay framework it
keeps the certificate's coverage (0.980--0.999 against 0.981--0.997 for
Poisson on four models) and lowers the median error by 8--16\%
(Table~\ref{tab:replayv2}), at the price of joint inclusion
probabilities that do not factor, so the Sen--Yates--Grundy form is
then an approximation whose coverage is measured, not proved.

\paragraph{Signals.}
Retained-attention entropy: decode-step attention entropy over the
retained set, normalized by $\log(\text{retained size})$, averaged as
above. Evicted score mass: one minus the retained share of the arm's own
importance mass at eviction time. Keep-boundary margin: minimum retained
unprotected score minus maximum evicted score, in units of the score
standard deviation. Output log-probability: mean generated-token log
probability. All signals are online: computable at serving time by that
arm.

\paragraph{Tasks and scoring.}
6k suite: LongBench HotpotQA, 2WikiMQA, MultiFieldQA-en, and
PassageRetrieval-en (100 examples each, deterministic shuffle), plus a
three-passcode needle anchor; contexts middle-truncated to 6{,}000
tokens. 16k suite: HotpotQA and PassageRetrieval-en under the same
sampling, plus MuSiQue and NarrativeQA; contexts to 16{,}000 tokens;
budgets $\{12.5\%, 25\%, 50\%\}$. QA scored by LongBench token-F1
against all references (success F1 $\ge 0.5$; the 0.3 threshold moves no
verdict); retrieval and needle by exact match. Chat prompts split the
template around the context so that stream arms compress before any
question token exists.

\paragraph{No-context control.}
Each instruct model answers each task's questions with the context
removed (same examples, same prompt scaffolding and scoring; 2{,}400
generations). No-context accuracy: 0.22--0.32 on HotpotQA and 2WikiMQA,
0.02--0.13 on MultiFieldQA-en, 0.07--0.13 on MuSiQue, 0.03--0.06 on
NarrativeQA, 0.00--0.05 on PassageRetrieval-en. Of full-cache-correct
examples, 75\% (6k) and 84\% (16k) are context-dependent. Restricted to
those examples, aware-condition induced-failure rates are 11.5/6.9/2.3\%
for 16k Poisson at budgets 12.5/25/50\% (unrestricted:
10.5/6.2/2.6\%) and 9.5/4.4\% for 6k Poisson at 25/50\% (8.5/4.0\%);
streaming H2O rises from 43.8\% to 47.3\% at 16k/25\%. The regime
picture of Figure~\ref{fig:damage} is unchanged.

\paragraph{Agent-memory system and triage.}
The vignette generates fifty ten-turn dialogues from twelve fact
templates (flight numbers, allergies, invoice numbers, door codes,
contact names) embedded in shuffled natural chatter, one fact per user
turn with a brief assistant acknowledgment; histories average 1{,}060
tokens. Compression uses the shared harness in the stream condition
(H2O-style prefill importance for the top-$k$ arm, observation-window
importance for the Poisson design), 25\% budget, protected sinks and
recency as elsewhere, two Poisson seeds; answers are scored by
substring match of the fact value. The failures differ in kind: asked
for the meeting room stated eight turns earlier, Qwen answers ``I don't
have that information on file,'' and Llama fabricates one (``the review
meeting is in room 519''; the user said B549); recency-window failures
are silent by construction. The budget sweep of Figure~\ref{fig:budget}
reuses the same dialogue generator, seeds, arms, and scoring at budgets
$\{20, 22, 24, 26, 28, 30\}\%$, merged for the operating picture with
the 10\% and 15\% cells of a six-budget grid run under identical seeds
(the shared 20\% cell reproduces bit-identically across the two runs);
the gated system is composed exactly from per-run records (greedy
decoding makes the full-cache re-answer identical to the full arm) and
the red-flag rule is fixed at $\tau = 1$ throughout; at budgets of 15\% and below the near-zero silent rate is the gated system's distinctive property. On Llama-3.1-8B the sweep reproduces every pattern: failure rate flat from 89\% to 80\%, red-flag coverage of failures 100\% to 8\%, failure-median certificate 1.48 to 0.83. The thirty-turn variant
extends the fact pool to 32 distinct
kinds (one per turn, no repeats), asks six recall questions per dialogue, and runs budgets $\{8, 10, 12, 15, 25\}\%$ with the same arms and scoring; the live transcripts (five seeds) are recorded at the 10\% and 15\% budgets, and Figure~\ref{fig:hero} shows the 15\% run of one of them. The scale replication runs the identical generator, seeds,
arms, and scoring on Qwen2.5-32B-Instruct at budgets $\{20, 25, 30\}\%$
(bf16, one H200). The fine-grained probe re-runs the 25\% budget on the
7B model, recording per-(layer, head) certificates over the first six
decode steps and the per-retained-token variance contribution; on each
red flag the question is re-answered four ways: full retrieval, evicted
tokens of 64-token blocks ranked by retained-neighbor contribution up
to 10\% of the context, random blocks at the same budget, and an oracle
that retrieves the block containing the queried fact value. The
3{,}200 layer--head matrices ship with the run records.
Figure~\ref{fig:agent}b sweeps each signal (certificate for Poisson,
retained entropy for H2O) over its own range per model, flags answers
above threshold, and averages the two models' curves at fixed
false-alarm rate; the marked operating point applies the fixed rule
$\mathrm{cert} \ge \tau = 1$ pooled over models, with no labels or
tuning involved. The triage analysis splits Poisson stream failures
into certificate terciles within each (model, task, budget) group, so
certificate scales never mix, and measures the fraction whose
full-cache run is correct. Among the third of stream failures the
certificate blames most, 32.4\% recover when rerun with the full cache;
among the third it exonerates, 13.7\% (6k suite, $n = 4{,}293$
failures, a 2.4$\times$ enrichment), while ranking the same failures
by output confidence is uninformative (22.8\% against 22.9\%); at 16k
the enrichment is 1.35$\times$ for the certificate against
1.11$\times$ for confidence.

\begin{figure}[t]
\centering
\includegraphics[width=0.72\linewidth]{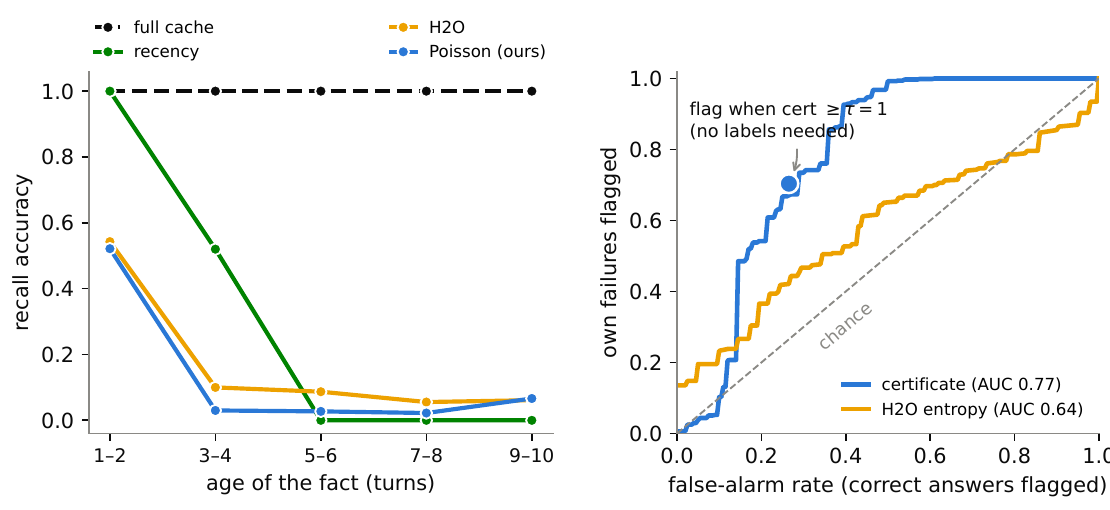}
\caption{Under streaming compression to a 25\% budget, facts three or
more turns old are lost by every importance-based rule, and only the
certificate flags the loss without labels. Two instruct models, 2{,}000
scored answers. (a) Recall by age of the fact. (b) Flagging one's own
failures, model-averaged; the marked point is the label-free rule
$\mathrm{cert} \ge \tau = 1$.}
\label{fig:agent}
\end{figure}

\paragraph{Real conversations.}
LoCoMo \citep{maharana2024locomo} provides ten two-person conversations
of 19--35 dated sessions with 1{,}986 questions in five categories; we
use the factual categories (single-hop, multi-hop, temporal; 699
questions). The transcript, with session headers and photo captions, is
ingested as chat history under a system prompt that asks for value-only
answers; each question is asked from the compressed views (recency, H2O,
SnapKV-style top-$k$, Poisson with certificate) at budgets
$\{5, 10, 15, 25\}\%$ and from the full archive. Success is the
reference contained in the answer or at least 75\% of its tokens
present (order-free, so date formats agree); the tables report mean
token F1. The gated system is composed from the records as elsewhere.

\paragraph{Statistics.}
AUCs are Mann--Whitney; intervals are 95\% cluster bootstrap resampling
examples (500 draws), so repeated measurements of one example never
inflate significance. Paired signal comparisons bootstrap the AUC
difference on shared examples. Pooled-across-model AUCs mix certificate
scales across models and are therefore conservative; per-model values
appear in Table~\ref{tab:permodel}.

\section{Pre-registration record}
\label{app:prereg}

The registration file was written on 2026-07-22 before any full 6k shard
completed, and its scale addendum before any 16k shard ran; both are
timestamped in the released package alongside the Slurm submission
records. Table~\ref{tab:verdicts} lists all seven claims with outcomes.

\begin{table}[h]
\centering
\small
\begin{tabular}{@{}llll@{}}
\toprule
& Claim (abbreviated) & Kill condition & Verdict \\
\midrule
R1 & Certificate predicts own task failure on & pooled AUC $< 0.6$ &
killed, both scales \\
& real tasks (deployable trust signal) & & (0.555 at 6k, 0.572 at 16k) \\
R2 & No deterministic self-signal sees its & any signal $\ge$ cert & survived
(max 0.73 vs.\ \\
& own eviction-induced failure & lower CI & cert 0.855; entropy at chance) \\
R3 & Certificate $\ge$ output logprob on & logprob strictly & not met (tie at
6k; \\
& induced-failure prediction & dominates & logprob wins at 16k) \\
R4 & Certificate-gated $25\%\!\to\!50\%$ escalation & cert $\le$ random &
killed (no headroom \\
& beats random at matched budget & & aware; $\approx$ random stream) \\
\midrule
S1 & 6k ``aware is free'' is a truncation & rate rise $< 5$pp & no rise:
free at 16k \\
& artifact; damage rises at 16k & confirms & (4.0\% at 25\%) \\
S2 & Certificate AUC recovers ($\ge 0.6$) at 16k & $< 0.6$ = final &
final kill (0.572) \\
S3 & Exploration tax reappears at 12.5\% & tax $< 3$pp confirms & no tax
($+0.3$pp) \\
\bottomrule
\end{tabular}
\caption{The seven pre-registered claims and their outcomes. The
prediction claims (R1, R3, R4, S2) failed and the regime and
attribution claims (R2, S1, S3) held, which is the division of labor
reported in Section~\ref{sec:real}.}
\label{tab:verdicts}
\end{table}

R1 (certificate validity transfers): pooled LongBench
certificate-to-own-failure AUC above 0.5 with CI excluding 0.5 in at
least 3 of 4 models; kill if pooled AUC $< 0.6$. Outcome: per-model
0.552--0.581 (6k), all above 0.5; pooled 0.555, kill triggered; 16k
pooled 0.572, kill confirmed final (S2).

R2 (silent failure is signal-general): every deterministic
retained-set signal at AUC $< 0.6$ or CI overlapping 0.5; kill if any
reaches the certificate's lower CI. Outcome: overall-failure panel
maxima 0.59 (evicted mass); induced-failure panel maxima 0.73 (6k
SnapKV evicted mass; 0.67 two-scale pooled), against certificate 0.855
(6k, lower CI 0.84); kill not triggered under either reading; the
partial visibility of evicted mass is reported.

R3 (certificate at least matches output confidence on induced
failures): kill if confidence strictly dominates. Outcome: 6k tie
($-0.004$, CI $[-0.031, +0.025]$); 16k confidence wins ($-0.038$, CI
$[-0.065, -0.010]$). Not met; prediction conceded to confidence.

R4 (certificate-gated escalation): kill if at or below random.
Outcome: killed; aware has no headroom (fixed-budget accuracies 0.538
against 0.539), stream certificate-gating matches random, because a
50\% stream cache is still broken. The recomputation analysis of
Section~\ref{sec:attribution} was designed after this autopsy and is
post hoc.

S1--S3 as in Table~\ref{tab:verdicts}: induced-failure rate at
16k/25\% is 4.0\% against 5.6\% at 6k/25\% (no rise; 12.5\% reaches
7.3\%); certificate AUC 0.572 ($< 0.6$, final); aware tax $+0.3$pp at
12.5\%.

Like the recomputation analysis, the per-seed replication, the exact
random-gating expectations, the two-draw disagreement baseline, and the
deployment experiments of Section~\ref{sec:vignettes} were added in
revision and are post hoc; the pre-registered claims and kill
conditions are exactly those listed above.

\section{Additional tables and figures}
\label{app:tables}

\begin{table}[h]
\centering
\footnotesize
\setlength{\tabcolsep}{3pt}
\begin{tabular}{@{}llcccccc@{}}
\toprule
model & attention & $L$ & coverage & Spearman & top-$k$ & Poisson+H\'ajek & Poisson, no offset \\
\midrule
SmolLM2-1.7B & GQA & 24 & 0.984 & 0.978 & 0.034 & \textbf{0.025} & 0.043 \\
Llama-3.2-3B & GQA & 28 & 0.995 & 0.976 & 0.074 & \textbf{0.039} & 0.092 \\
Qwen2.5-7B & GQA & 28 & 0.981 & 0.973 & 0.049 & \textbf{0.033} & 0.063 \\
Llama-3.1-8B & GQA & 32 & 0.995 & 0.984 & 0.054 & \textbf{0.032} & 0.068 \\
Mistral-7B-v0.3 & GQA & 32 & 0.997 & 0.984 & 0.068 & \textbf{0.039} & 0.085 \\
Falcon3-7B & GQA & 28 & 0.996 & 0.988 & 0.046 & \textbf{0.027} & 0.057 \\
Qwen3-8B & QK-norm & 36 & 0.995 & 0.980 & 0.059 & \textbf{0.040} & 0.074 \\
OLMo-2-7B & QK-norm & 32 & 0.997 & 0.975 & 0.179 & \textbf{0.081} & 0.210 \\
Gemma-2-9B & soft-cap & 42 & 0.993 & 0.983 & 0.033 & \textbf{0.018} & 0.042 \\
phi-4 & fused qkv & 40 & 0.996 & 0.984 & 0.061 & \textbf{0.043} & 0.076 \\
Qwen2.5-14B & GQA & 48 & 0.995 & 0.981 & 0.062 & \textbf{0.043} & 0.080 \\
Qwen1.5-MoE-A2.7B & GQA & 24 & 0.993 & 0.982 & 0.041 & \textbf{0.028} & 0.051 \\
\bottomrule
\end{tabular}
\caption{The certificate is architecture-agnostic: on twelve models spanning GQA, per-head and full-width QK-norm, fused qkv projections, logit soft-capping and mixture-of-experts feed-forward layers, coverage of the realized attention error at the 25\% budget is 0.981--0.997 and the Spearman correlation between certificate and error is 0.973--0.988. Median relative attention error at the 25\% budget: Poisson with the H\'ajek offset is below top-$k$ on every model, and Poisson without the offset is above top-$k$ on every model. Replay protocol of Section~\ref{sec:replay}, 10{,}368--20{,}736 cells per model.}
\label{tab:zoo}
\end{table}

\begin{table}[h]
\centering
\scriptsize
\setlength{\tabcolsep}{3pt}
\resizebox{\linewidth}{!}{\begin{tabular}{@{}llccccccc@{}}
\toprule
model & budget & resident & top-$k$ & Poisson & gated & red-flag & silent & attribution AUC \\
 & & tokens & & & ($\tau=1$) & rate & & cert / logprob \\
\midrule
Qwen2.5-7B-Instruct (10 conv.) & full & -- & \multicolumn{2}{c}{0.297} & & & & \\
 & 2\% & 424 & 0.092 & 0.068 & \textbf{0.297} & 1.00 & 0.00 & 0.58 / 0.48 \\
 & 3\% & 630 & 0.090 & 0.093 & \textbf{0.297} & 1.00 & 0.00 & 0.54 / 0.52 \\
 & 5\% & 1,053 & 0.110 & 0.089 & \textbf{0.252} & 0.79 & 0.20 & 0.54 / 0.50 \\
 & 10\% & 2,110 & 0.130 & 0.140 & \textbf{0.216} & 0.48 & 0.48 & 0.55 / 0.51 \\
 & 15\% & 3,164 & 0.153 & 0.155 & \textbf{0.179} & 0.14 & 0.80 & 0.53 / 0.48 \\
 & 25\% & 5,308 & 0.192 & 0.206 & \textbf{0.206} & 0.00 & 0.92 & 0.58 / 0.44 \\
\midrule
Qwen2.5-32B-Instruct (10 conv.) & full & -- & \multicolumn{2}{c}{0.397} & & & & \\
 & 5\% & 1,051 & 0.090 & 0.131 & \textbf{0.390} & 0.96 & 0.04 & 0.55 / 0.48 \\
 & 10\% & 2,110 & 0.157 & 0.181 & \textbf{0.289} & 0.52 & 0.44 & 0.55 / 0.45 \\
 & 15\% & 3,163 & 0.194 & 0.209 & \textbf{0.245} & 0.18 & 0.76 & 0.54 / 0.46 \\
 & 25\% & 5,308 & 0.265 & 0.255 & \textbf{0.255} & 0.00 & 0.89 & 0.60 / 0.49 \\
\midrule
Llama-3.1-8B-Instruct (10 conv.) & full & -- & \multicolumn{2}{c}{0.327} & & & & \\
 & 2\% & 427 & 0.063 & 0.067 & \textbf{0.316} & 0.97 & 0.03 & 0.55 / 0.52 \\
 & 3\% & 630 & 0.072 & 0.088 & \textbf{0.255} & 0.68 & 0.30 & 0.54 / 0.46 \\
 & 5\% & 1,043 & 0.086 & 0.098 & \textbf{0.185} & 0.33 & 0.64 & 0.55 / 0.48 \\
 & 10\% & 2,100 & 0.121 & 0.159 & \textbf{0.167} & 0.02 & 0.92 & 0.55 / 0.47 \\
 & 15\% & 3,173 & 0.157 & 0.183 & \textbf{0.183} & 0.00 & 0.93 & 0.57 / 0.47 \\
 & 25\% & 5,287 & 0.194 & 0.227 & \textbf{0.227} & 0.00 & 0.91 & 0.59 / 0.55 \\
\bottomrule
\end{tabular}}
\caption{Real long-term conversations (LoCoMo, ten conversations of 11k--25k tokens, factual questions, stream condition, mean token F1). The heavy-damage regime sits at a resident cache of roughly 400--1{,}000 tokens on every model (2--3\% of the history for Llama, 3--5\% for Qwen); there the certificate-gated system returns the full-cache score while top-$k$ keeps a third of it, and the 32B model raises the ceiling without changing the certificate's scale (median 1.23 and 1.25 at 5\% for the two Qwen sizes). Above that point the flag goes dark and the residual failures are inherent or retrieval-type. The certificate separates eviction-induced from inherent failures better than output log-probability at every budget on every model.}
\label{tab:locomo}
\end{table}

\begin{table}[h]
\centering
\small
\setlength{\tabcolsep}{4pt}
\resizebox{\linewidth}{!}{\begin{tabular}{@{}lcccccc@{}}
\toprule
model & \multicolumn{2}{c}{$\tau=1$ crossing (ten turns)} & gated / top-$k$ & gated / top-$k$ & AUC & crossing \\
 & budget & resident tokens & at 10\% & at 15\% & at 15\% & (thirty turns) \\
\midrule
Qwen2.5-7B-Instruct & 25\% & 264 & 0.98 / 0.10 & 0.98 / 0.12 & 0.92 & 14\% (457 tokens) \\
Llama-3.1-8B-Instruct & 21\% & 226 & 0.98 / 0.12 & 0.98 / 0.15 & 0.84 & -- \\
Qwen3-8B & 21\% & 225 & 0.98 / 0.12 & 0.87 / 0.14 & 0.77 & 11\% (345 tokens) \\
phi-4 & 22\% & 225 & 0.99 / 0.11 & 0.95 / 0.14 & 0.73 & 11\% (337 tokens) \\
Qwen2.5-32B-Instruct & 29\% & 308 & -- / -- & -- / -- & -- & -- \\
\bottomrule
\end{tabular}}
\caption{The red-flag boundary across model families in the certified configuration (ten-turn vignette, fifty dialogues, stream condition; thirty-turn where run). The budget at which the median certificate crosses $\tau = 1$ differs by family, but the resident cache at the crossing stays within a few hundred tokens; inside the heavy-damage regime (10\%) the gated system answers 0.9--1.0 against 0.1--0.2 for top-$k$ on every model.}
\label{tab:families}
\end{table}

\begin{table}[h]
\centering
\small
\setlength{\tabcolsep}{4pt}
\resizebox{\linewidth}{!}{\begin{tabular}{@{}lcccccc@{}}
\toprule
& \multicolumn{3}{c}{25\% budget} & \multicolumn{3}{c}{15\% budget} \\
setting & AUC & red-flag rate & gated acc.\ & AUC & red-flag rate & gated acc.\ \\
\midrule
deployed ($\delta=0.1$, 6 steps, 4 heads/layer, mean, max over steps) & 0.802 & 0.47 & 0.575 & 0.877 & 0.95 & 1.000 \\
$\delta = 0.05$ & 0.801 & 0.89 & 0.980 & 0.877 & 0.99 & 1.000 \\
$\delta = 0.2$ & 0.803 & 0.03 & 0.165 & 0.878 & 0.48 & 0.575 \\
1 decode step & 0.815 & 0.01 & 0.150 & 0.872 & 0.76 & 0.840 \\
3 decode steps & 0.816 & 0.34 & 0.465 & 0.906 & 0.93 & 0.995 \\
12 decode steps & 0.788 & 0.64 & 0.735 & 0.893 & 0.95 & 1.000 \\
1 head per layer & 0.773 & 0.45 & 0.555 & 0.856 & 0.95 & 0.995 \\
all 28 heads & 0.782 & 0.47 & 0.585 & 0.873 & 0.94 & 0.990 \\
median over units & 0.810 & 0.20 & 0.325 & 0.881 & 0.87 & 0.960 \\
max over units & 0.403 & 1.00 & 1.000 & 0.465 & 1.00 & 1.000 \\
mean over steps & 0.797 & 0.10 & 0.235 & 0.893 & 0.89 & 0.975 \\
\bottomrule
\end{tabular}}
\caption{Implementation choices of the certificate, recomputed offline from per-unit traces of the ten-turn vignette (Qwen2.5-7B-Instruct, 200 answers per budget, offset on). The failure-ranking AUC is stable across confidence level, probed steps, head count and aggregation rule (0.77--0.82 at 25\%, 0.85--0.94 at 15\%); the operating point of the label-free rule $\mathrm{cert} \ge 1$ moves with $\delta$ and the number of probed steps, which scale the radius, and max-over-units saturates.}
\label{tab:certabl}
\end{table}

\begin{table}[h]
\centering
\small
\setlength{\tabcolsep}{4pt}
\resizebox{\linewidth}{!}{\begin{tabular}{@{}llccccc|c@{}}
\toprule
budget, thresholds & gate & 5\% & 10\% & 15\% & 25\% & 50\% & rate to match random at 25\% \\
\midrule
12.5\%, one global threshold & certificate & \textbf{0.259} & \textbf{0.276} & \textbf{0.290} & \textbf{0.320} & \textbf{0.382} & 15.7\% \\
 & log-probability & 0.247 & 0.255 & 0.265 & 0.286 & 0.332 & 30.4\% \\
 & random & 0.250 & 0.260 & 0.271 & 0.292 & 0.345 & -- \\
\midrule
12.5\%, per-task thresholds & certificate & 0.256 & \textbf{0.269} & 0.283 & 0.308 & 0.364 & 19.5\% \\
 & log-probability & \textbf{0.257} & \textbf{0.269} & \textbf{0.284} & \textbf{0.313} & \textbf{0.374} & 17.9\% \\
 & random & 0.250 & 0.260 & 0.271 & 0.292 & 0.345 & -- \\
\midrule
25\%, one global threshold & certificate & \textbf{0.355} & \textbf{0.367} & \textbf{0.372} & \textbf{0.389} & \textbf{0.417} & 11.0\% \\
 & log-probability & 0.344 & 0.348 & 0.354 & 0.366 & 0.391 & 27.5\% \\
 & random & 0.346 & 0.351 & 0.357 & 0.368 & 0.395 & -- \\
\midrule
25\%, per-task thresholds & certificate & \textbf{0.354} & \textbf{0.363} & \textbf{0.372} & 0.385 & 0.415 & 12.8\% \\
 & log-probability & 0.352 & 0.362 & \textbf{0.372} & \textbf{0.391} & \textbf{0.422} & 13.0\% \\
 & random & 0.346 & 0.351 & 0.357 & 0.368 & 0.395 & -- \\
\bottomrule
\end{tabular}}
\caption{Iso-accuracy view of recomputation scheduling on the 16k LongBench suite (stream Poisson arms, four models, four tasks, two seeds). With one threshold serving the mixed workload, the certificate reaches the accuracy of random gating at a quarter of the queries by re-running 9.7--14.0\% of them, and log-probability gating needs more re-runs than random; with a separate threshold per model and task, log-probability gating catches up. The certificate is a relative error bound whose scale transfers across tasks; a confidence score's scale does not, so it needs per-task calibration to schedule well.}
\label{tab:iso}
\end{table}

\begin{table}[h]
\centering
\small
\setlength{\tabcolsep}{3pt}
\resizebox{\linewidth}{!}{\begin{tabular}{@{}lc|ccc|cc|cc@{}}
\toprule
 & & \multicolumn{3}{c|}{median relative error} & \multicolumn{2}{c|}{Spearman(merge error, $\cdot$)} & \multicolumn{2}{c}{coverage} \\
model & budget & top-$k$ & merge & Poisson & dispersion & evicted mass & Poisson & fixed-size \\
\midrule
Llama-3.1-8B-Instruct & 50\% & 0.022 & 0.021 & 0.011 & 0.04 & 0.81 & 0.995 & 0.994 \\
 & 25\% & 0.065 & 0.063 & 0.037 & 0.08 & 0.75 & 0.994 & 0.995 \\
 & 12.5\% & 0.126 & 0.122 & 0.091 & 0.04 & 0.70 & 0.993 & 0.994 \\
\midrule
Qwen2.5-7B-Instruct & 50\% & 0.014 & 0.013 & 0.008 & $-0.01$ & 0.92 & 0.987 & 0.989 \\
 & 25\% & 0.050 & 0.047 & 0.035 & 0.05 & 0.88 & 0.984 & 0.986 \\
 & 12.5\% & 0.113 & 0.104 & 0.089 & 0.04 & 0.83 & 0.981 & 0.980 \\
\midrule
Qwen3-8B & 50\% & 0.017 & 0.017 & 0.010 & $-0.07$ & 0.87 & 0.993 & 0.996 \\
 & 25\% & 0.059 & 0.056 & 0.041 & $-0.06$ & 0.82 & 0.994 & 0.996 \\
 & 12.5\% & 0.124 & 0.120 & 0.103 & $-0.06$ & 0.76 & 0.991 & 0.993 \\
\midrule
phi-4 & 50\% & 0.019 & 0.018 & 0.011 & 0.06 & 0.88 & 0.997 & 0.999 \\
 & 25\% & 0.061 & 0.060 & 0.043 & 0.10 & 0.82 & 0.996 & 0.997 \\
 & 12.5\% & 0.125 & 0.125 & 0.106 & 0.07 & 0.76 & 0.995 & 0.996 \\
\bottomrule
\end{tabular}}
\caption{Merging and fixed-size designs in the replay framework (six prompts, 48 probes per layer, every second layer, bf16). Merging error tracks the within-stratum dispersion of evicted values more closely than the evicted attention mass, as Remark~\ref{rem:methods} predicts, and the Sen--Yates--Grundy certificate keeps its coverage under fixed-size systematic sampling at matched inclusion probabilities.}
\label{tab:replayv2}
\end{table}

\begin{table}[h]
\centering
\small
\setlength{\tabcolsep}{3pt}
\resizebox{\linewidth}{!}{\begin{tabular}{@{}llccccccc@{}}
\toprule
model & task & budget & full & SnapKV & PyramidKV & CaM-merge & Poisson & gated ($\tau=1$) \\
\midrule
Llama-3.1-8B-Instruct & needle & 25\% & 1.00 & 0.36 & 0.36 & \textbf{0.41} & 0.11 & 0.16 \\
 & needle & 50\% & 1.00 & \textbf{0.88} & 0.85 & \textbf{0.88} & 0.57 & 0.57 \\
 & 3 needles & 25\% & 1.00 & 0.12 & \textbf{0.15} & \textbf{0.15} & 0.03 & 0.14 \\
 & 3 needles & 50\% & 1.00 & 0.70 & \textbf{0.76} & 0.74 & 0.39 & 0.39 \\
 & kv retrieval & 25\% & 1.00 & 0.00 & 0.00 & 0.00 & 0.00 & \textbf{0.09} \\
 & kv retrieval & 50\% & 1.00 & 0.03 & \textbf{0.04} & 0.01 & 0.01 & 0.01 \\
 & var.\ tracking & 25\% & 0.10 & \textbf{0.01} & 0.00 & \textbf{0.01} & 0.00 & 0.00 \\
 & var.\ tracking & 50\% & 0.10 & 0.01 & 0.01 & 0.03 & \textbf{0.09} & \textbf{0.09} \\
\midrule
Qwen2.5-7B-Instruct & needle & 25\% & 1.00 & 0.00 & 0.00 & 0.00 & 0.00 & \textbf{0.57} \\
 & needle & 50\% & 1.00 & 0.10 & \textbf{0.12} & 0.10 & 0.03 & 0.03 \\
 & 3 needles & 25\% & 1.00 & 0.00 & 0.00 & 0.00 & 0.00 & \textbf{0.80} \\
 & 3 needles & 50\% & 1.00 & 0.04 & 0.04 & 0.04 & 0.01 & \textbf{0.06} \\
 & kv retrieval & 25\% & 0.97 & 0.00 & 0.00 & 0.00 & 0.00 & \textbf{0.97} \\
 & kv retrieval & 50\% & 0.97 & 0.00 & 0.00 & 0.00 & 0.00 & \textbf{0.07} \\
 & var.\ tracking & 25\% & 0.04 & 0.00 & 0.00 & 0.00 & 0.00 & \textbf{0.03} \\
 & var.\ tracking & 50\% & 0.04 & 0.00 & 0.00 & 0.00 & \textbf{0.01} & \textbf{0.01} \\
\midrule
Qwen3-8B & needle & 25\% & 1.00 & 0.01 & 0.04 & 0.01 & 0.00 & \textbf{0.33} \\
 & needle & 50\% & 1.00 & 0.10 & \textbf{0.25} & 0.15 & 0.05 & 0.05 \\
 & 3 needles & 25\% & 1.00 & 0.00 & 0.00 & 0.00 & 0.00 & \textbf{0.66} \\
 & 3 needles & 50\% & 1.00 & 0.00 & \textbf{0.06} & 0.01 & 0.00 & 0.01 \\
 & kv retrieval & 25\% & 1.00 & 0.00 & 0.00 & 0.00 & 0.00 & \textbf{0.64} \\
 & kv retrieval & 50\% & 1.00 & 0.00 & 0.00 & 0.00 & 0.00 & 0.00 \\
 & var.\ tracking & 25\% & 0.80 & 0.00 & 0.00 & 0.00 & 0.00 & \textbf{0.32} \\
 & var.\ tracking & 50\% & 0.80 & 0.00 & \textbf{0.04} & 0.00 & 0.00 & 0.00 \\
\midrule
phi-4 & needle & 25\% & 1.00 & 0.17 & 0.23 & 0.20 & 0.09 & \textbf{0.51} \\
 & needle & 50\% & 1.00 & 0.85 & 0.82 & \textbf{0.89} & 0.42 & 0.42 \\
 & 3 needles & 25\% & 1.00 & 0.00 & 0.01 & 0.03 & 0.04 & \textbf{0.68} \\
 & 3 needles & 50\% & 1.00 & \textbf{0.46} & 0.45 & \textbf{0.46} & 0.16 & 0.16 \\
 & kv retrieval & 25\% & 1.00 & 0.00 & 0.00 & 0.00 & 0.00 & \textbf{0.81} \\
 & kv retrieval & 50\% & 1.00 & \textbf{0.04} & 0.03 & \textbf{0.04} & 0.03 & 0.03 \\
 & var.\ tracking & 25\% & 0.31 & 0.06 & 0.01 & 0.07 & 0.04 & \textbf{0.23} \\
 & var.\ tracking & 50\% & 0.31 & 0.23 & 0.03 & \textbf{0.25} & 0.12 & 0.12 \\
\midrule
mean & all & 25\% & 0.83 & 0.05 & 0.05 & 0.06 & 0.02 & 0.43 \\
mean & all & 50\% & 0.83 & 0.21 & 0.22 & 0.22 & 0.12 & 0.12 \\
\bottomrule
\end{tabular}}
\caption{Deterministic baselines with budget allocation (PyramidKV, layer-wise budgets from 1.5$\times$ to 0.5$\times$ the mean) and compensation (CaM-style merging of evicted values into the nearest retained key) on the synthetic suite, stream condition, certified configuration. Bold marks the best compressed system in each row. Allocation and merging change which tokens go and what the retained values hold; neither produces an estimate of the eviction error, and the certificate-gated system stays ahead where the damage is heavy.}
\label{tab:base}
\end{table}

\begin{table}[h]
\centering
\small
\begin{tabular}{@{}lcccc@{}}
\toprule
Model & \multicolumn{2}{c}{6k} & \multicolumn{2}{c}{16k} \\
& AUC & 95\% CI & AUC & 95\% CI \\
\midrule
Qwen2.5-1.5B-Instruct & 0.581 & [0.555, 0.609] & 0.593 & [0.567, 0.621] \\
Qwen2.5-7B-Instruct & 0.558 & [0.535, 0.581] & 0.602 & [0.583, 0.622] \\
Llama-3.1-8B-Instruct & 0.565 & [0.545, 0.587] & 0.582 & [0.567, 0.599] \\
Mistral-7B-Instruct-v0.3 & 0.552 & [0.532, 0.574] & 0.596 & [0.580, 0.613] \\
\midrule
pooled & 0.555 & [0.543, 0.568] & 0.572 & [0.562, 0.582] \\
\bottomrule
\end{tabular}
\caption{R1 detail: certificate-to-own-failure AUC, LongBench only.
Every model exceeds chance; none reaches the pre-registered 0.6 line.
Cell-level: 24/32 cells above 0.5 at 6k, 30/32 at 16k.}
\label{tab:permodel}
\end{table}

\begin{table}[h]
\centering
\scriptsize
\setlength{\tabcolsep}{3.5pt}
\begin{tabular}{@{}lccccc@{}}
\toprule
Arm & certificate (ours) & entropy & evicted mass & margin & output logprob
(baseline) \\
\midrule
\multicolumn{6}{@{}l}{\emph{Overall failure (6k suite, pooled)}} \\
StreamingLLM & -- & .472 [.45,.50] & -- & -- & \textbf{.822} [.80,.84] \\
H2O & -- & .521 [.50,.54] & .590 [.58,.61] & .503 [.50,.51] &
\textbf{.772} [.75,.79] \\
SnapKV & -- & .538 [.51,.56] & .592 [.58,.61] & .502 [.50,.51] &
\textbf{.771} [.75,.79] \\
aware top-$k$ & -- & .580 [.56,.61] & .502 [.49,.52] & .499 [.50,.50] &
\textbf{.799} [.78,.82] \\
Poisson & .645 [.63,.66] & .527 [.50,.55] & -- & -- & \textbf{.794} [.78,.81] \\
\midrule
\multicolumn{6}{@{}l}{\emph{Eviction-induced failure (full-correct
examples, 6k suite)}} \\
StreamingLLM & -- & .455 [.42,.49] & -- & -- & \textbf{.861} [.84,.88] \\
H2O & -- & .501 [.47,.53] & .660 [.64,.68] & .500 [.49,.51] &
\textbf{.831} [.81,.85] \\
SnapKV & -- & .509 [.48,.54] & .726 [.71,.74] & .501 [.50,.51] &
\textbf{.834} [.82,.85] \\
aware top-$k$ & -- & .699 [.61,.78] & .575 [.51,.65] & .493 [.48,.50] &
\textbf{.852} [.81,.89] \\
Poisson & \textbf{.855} [.84,.87] & .446 [.43,.47] & -- & -- & .848 [.83,.86] \\
\bottomrule
\end{tabular}
\caption{The full self-signal panel: AUC predicting the arm's own
failure, with 95\% cluster-bootstrap intervals (resampling examples);
best signal per row in bold. Output log-probability is the comparison
baseline and is strong everywhere, which is why prediction is conceded
to it; no retained-set signal of a deterministic arm approaches the
certificate on induced failures.}
\label{tab:panel}
\end{table}

\begin{table}[h]
\centering
\scriptsize
\setlength{\tabcolsep}{4pt}
\begin{tabular}{@{}llccccccc@{}}
\toprule
Suite/budget & seed & base & full & cert (ours) & logprob & random
(baseline) & cert$-$random [95\% CI] & cert$-$logprob [95\% CI] \\
\midrule
6k / 25\% & 0 & .296 & .436 & \textbf{.358} & .346 & .331 &
$+.027$ [$+.015,+.039$] & $+.013$ [$-.004,+.030$] \\
 & 1 & .284 & .436 & \textbf{.352} & .338 & .322 &
$+.030$ [$+.019,+.043$] & $+.014$ [$-.003,+.033$] \\
16k / 12.5\% & 0 & .240 & .450 & \textbf{.329} & .281 & .292 &
$+.037$ [$+.023,+.049$] & $+.049$ [$+.031,+.066$] \\
 & 1 & .247 & .450 & \textbf{.334} & .291 & .298 &
$+.037$ [$+.024,+.050$] & $+.044$ [$+.026,+.062$] \\
16k / 25\% & 0 & .343 & .450 & \textbf{.389} & .372 & .369 &
$+.019$ [$+.010,+.030$] & $+.016$ [$+.002,+.032$] \\
 & 1 & .346 & .450 & \textbf{.394} & .367 & .372 &
$+.022$ [$+.011,+.032$] & $+.027$ [$+.010,+.043$] \\
\bottomrule
\end{tabular}
\caption{Certificate gating recovers the most accuracy in every
suite, budget, and seed. Gated recomputation (stream condition,
LongBench tasks only, 25\% re-run rate) with per-seed replication.
Random gating is the exact expectation $0.25\,\mathrm{acc}_{\mathrm{full}} +
0.75\,\mathrm{acc}_{\mathrm{base}}$; intervals are cluster bootstrap
over examples. The two-draw attributor (partner-draw success and the
F1 gap between draws, label-assisted proxies) scores AUC 0.632 and
0.687 at 6k and 0.626 and 0.620 at 16k against the certificate's 0.749
and 0.727 from a single draw; recomputation gated on two-draw
disagreement (fires on 15\% of queries) reaches .314/.274/.373 in the
three suite/budget rows and trails certificate gating at the matched
rate (.331--.343/.289--.306/.376--.379) in every configuration. At
natural 23k-token lengths the full cache answers 27\%, the headroom
shrinks to ten points, and at $n=200$ no gating signal separates from
random: budgeted recomputation is a product for regimes with
headroom.}
\label{tab:dprime}
\end{table}

\begin{figure}[h]
\centering
\includegraphics[width=0.75\linewidth]{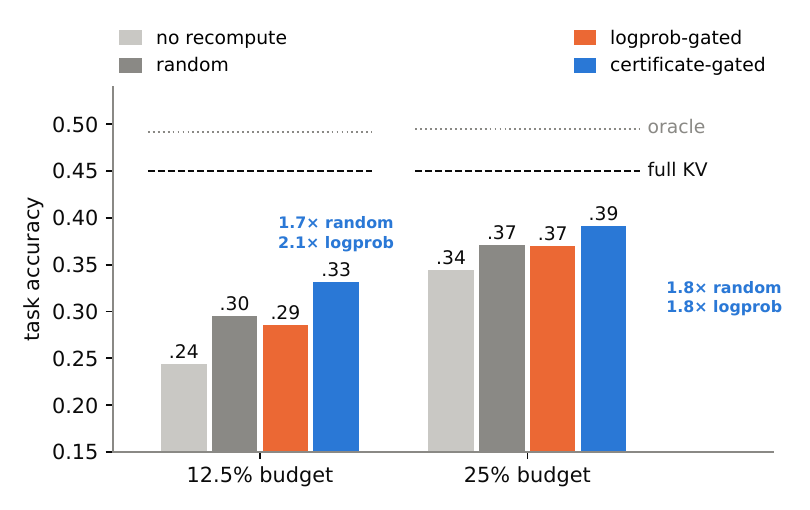}
\caption{Certificate gating recovers more accuracy than random or
confidence gating at matched compute. Streaming compression at 16k,
25\% of queries re-run with the full cache; bars average the two
Poisson seeds, and random gating is scored by its exact expectation.}
\label{fig:recompute}
\end{figure}

\begin{table}[h]
\centering
\small
\begin{tabular}{@{}lccccccc@{}}
\toprule
Task (16k) & full & H2O@25\% & H2O@50\% & aware@25\% & Poisson@25\% &
Poisson@50\% \\
\midrule
HotpotQA & 0.515 & 0.420 & 0.470 & \textbf{0.530} & 0.516 & 0.524 \\
MuSiQue & 0.260 & 0.180 & 0.215 & 0.258 & \textbf{0.263} & 0.254 \\
NarrativeQA & 0.205 & 0.147 & 0.195 & \textbf{0.200} & 0.198 & 0.196 \\
PassageRetrieval & 0.820 & 0.445 & 0.738 & 0.818 & 0.815 & \textbf{0.821} \\
\bottomrule
\end{tabular}
\caption{Question-aware compression tracks the full cache and streaming
compression does not. Accuracy by task at 16k, pooled over models
(aware condition for aware/Poisson columns; H2O is the stream-condition
baseline); best compressed arm per row in bold.}
\label{tab:acc16k}
\end{table}

\begin{table}[h]
\centering
\small
\begin{tabular}{@{}lcccccc@{}}
\toprule
& \multicolumn{2}{c}{AUC (failure)} & \multicolumn{2}{c}{gated accuracy}
& \multicolumn{2}{c}{certificate median} \\
budget & 7B & 32B & 7B & 32B & 7B & 32B \\
\midrule
20\% & 0.866 & 0.800 & 0.885 & 0.975 & 1.11 & 1.24 \\
25\% & 0.841 & 0.742 & 0.600 & 0.823 & 0.99 & 1.09 \\
30\% & 0.799 & 0.695 & 0.347 & 0.578 & 0.89 & 0.97 \\
\bottomrule
\end{tabular}
\caption{The certificate survives a 4.6$\times$ scale-up: AUC drops by 0.07--0.10, certificate medians rise, the red flag fires more often (0.92/0.71/0.43 against 0.80/0.48/0.19), and gated accuracy improves at every budget. Qwen2.5-7B against Qwen2.5-32B on the ten-turn vignette, identical dialogues, seeds, and scoring, offset on; $\tau = 1$ is unchanged.}
\label{tab:scale}
\end{table}

\begin{table}[h]
\centering
\small
\begin{tabular}{@{}lcc@{}}
\toprule
re-answer policy & retrieved & accuracy \\
\midrule
full archive & 100\% & 1.000 \\
oracle (fact block first) & 10\% & 0.894 \\
contribution-guided blocks (ours) & 10\% & \textbf{0.702} \\
random blocks (baseline) & 10\% & 0.181 \\
\bottomrule
\end{tabular}
\caption{Certificate-guided retrieval of 10\% of the context recovers 3.9$\times$ the accuracy of random retrieval at the same budget. The 94 red-flagged answers of the fine-grained probe (7B, 25\% budget, offset on); evicted 64-token blocks ranked by the variance contribution of their retained neighbors, which places the block holding the queried fact first in 64\% of cases and in the top three in 79\% (of roughly seventeen blocks).}
\label{tab:finegrain}
\end{table}

\end{document}